\documentclass[sts,preprint]{imsart}
\RequirePackage{amsthm,amsmath,amsfonts,amssymb}
\RequirePackage[authoryear]{natbib}

\startlocaldefs
\usepackage[utf8]{inputenc} 
\usepackage[T1]{fontenc}    
\usepackage{url}            
\usepackage{booktabs}       
\usepackage{amsfonts,amssymb,amsthm,mathtools}
\usepackage{nicefrac,xfrac} 
\usepackage{microtype}      
\usepackage{graphicx,subcaption,enumitem}
\usepackage[ruled]{algorithm2e}
\usepackage{algorithmic}
\usepackage[table]{xcolor}

\definecolor{green}{rgb}{0.0, 0.5, 0.0}
\definecolor{red}{rgb}{0.8, 0.0, 0.0}
\definecolor{blue}{rgb}{0.01, 0.28, 1.0}
\definecolor{yellow}{rgb}{0.98, 0.93, 0.36}

\definecolor{high}{RGB}{213,94,0}
\definecolor{greenish}{RGB}{94,213,0}
\definecolor{cite}{RGB}{230,159,0}
\usepackage[colorlinks = blue,
            linkcolor = blue,
            urlcolor  = blue,
            citecolor = blue,
            anchorcolor = blue]{hyperref}
\usepackage{cleveref}
\usepackage{stackengine}
\usepackage{bm}

\usepackage[textwidth=2.0cm, textsize=tiny]{todonotes}
\newcommand{\evg}[2][]{\todo[color=yellow!50,#1]{{\bf Evg:} #2}}

\newcommand{\moh}[2][]{\todo[color=green!20,#1]{{\bf Moh:} #2}}

\newcommand{\fadedmidrule}{\arrayrulecolor{black!10}\midrule\arrayrulecolor{black}}

\newcommand{\bsX}{{\boldsymbol X}}
\newcommand{\bsp}{{\boldsymbol p}}
\newcommand{\bsx}{{\boldsymbol x}}
\newcommand{\bso}{{\boldsymbol o}}
\newcommand{\bsY}{{\boldsymbol Y}}
\newcommand{\bsy}{{\boldsymbol y}}
\newcommand{\bsh}{{\boldsymbol h}}

\DeclareMathOperator*{\argmin}{arg\,min}
\DeclareMathOperator*{\argmax}{arg\,max}
\newcommand{\Exp}{\mathbb{E}} 
\newcommand{\Prob}{\mathbb{P}} 
\newcommand{\bbR}{\mathbb{R}} 
\newcommand{\bbN}{\mathbb{N}} 

\newcommand{\class}[1]{\mathcal{#1}} 
\newcommand{\enscond}[2]{\left\{#1 \,:\, #2\right\}}
\newcommand{\ens}[1]{\left\{#1 \right\}}
\newcommand{\ind}[1]{\mathbb{I}{\left\{#1\right\}}}
\newcommand{\recall}{\text{Rec}}
\newcommand{\precision}{\text{Prec}}

\newcommand{\eqdef}{\vcentcolon=}

\newcommand{\ie}{{\em i.e.,~}}

\newcommand{\eg}{{\em e.g.,~}}

\newcommand{\wrt}{{\em w.r.t.~}}
\newcommand{\iid}{{\rm i.i.d.~}}

\newcommand{\simiid}{\overset{\text{\iid}}{\sim}}
\DeclareMathOperator{\Top}{top}
\newtheorem{lemma}{Lemma}
\newtheorem{assumption}{Assumption}

\DeclareMathOperator{\size}{S}
\DeclareMathOperator{\error}{Err}

\newcommand{\dataset}[1]{\texttt{{#1}}}

\newtagform{bold}[\bfseries]{(}{\mdseries)}

\newcommand{\myparagraph}[1]{\noindent \textbf{\textit{#1.}}}

\def\Pr{{\mathbb{P}}}

\def\X{{\mathcal{X}}}

\def\cC{{\mathcal{C}}}

\DeclareMathOperator{\objective}{F_\Pr}
\DeclareMathOperator{\constraint}{\cC_\Pr}

\newcommand{\myref}[1]{{\protect\hypersetup{linkcolor=high}\ref{#1}}}

\newlist{prosandcons}{itemize}{1}
\setlist[prosandcons]{
  leftmargin=50pt,
  itemsep=0.5pt,
}

\endlocaldefs

\begin{document}

\begin{frontmatter}
\title{Set-valued classification -- overview via a unified framework}
\runtitle{Set-valued classification}







\begin{aug}
\author[A]{\fnms{Evgenii} \snm{Chzhen}\ead[label=e1]{evgenii.chzhen@universite-paris-saclay.fr}},
\author[B]{\fnms{Christophe} \snm{Denis}\ead[label=e2]{christophe.denis@u-pem.fr}},
\author[B]{\fnms{Mohamed} \snm{Hebiri}\ead[label=e3]{mohamed.hebiri@u-pem.fr}}
\and
\author[D]{\fnms{Titouan} \snm{Lorieul}\ead[label=e4]{titouan.lorieul@inria.fr}}
\address[A]{LMO, Universit\'e Paris-Saclay, CNRS, INRIA, \printead{e1}.}
\address[B]{LAMA, Universit\'e Gustave-Eiffel, \printead{e2,e3}.}
\address[C]{INRIA, ZENITH, Universit\'e de Montpellier, \printead{e4}.}
\end{aug}


\begin{abstract}
    Multi-class classification problem is among the most popular and well-studied statistical frameworks.
    Modern multi-class datasets can be extremely ambiguous and single-output predictions fail to deliver satisfactory performance.
    By allowing predictors to predict a set of label candidates, set-valued classification offers a natural way to deal with this ambiguity.
    Several formulations of set-valued classification are available in the literature and each of them leads to different prediction strategies.
    The present survey aims to review popular formulations using a unified statistical framework.
    The proposed framework encompasses previously considered and leads to new formulations as well as it allows to understand underlying trade-offs of each formulation.
    We provide infinite sample optimal set-valued classification strategies and review a general plug-in principle to construct data-driven algorithms.
    The exposition is supported by examples and pointers to both theoretical and practical contributions.
    Finally, we provide experiments on real-world datasets comparing these approaches in practice and providing general practical guidelines.
\end{abstract}


\end{frontmatter}

\maketitle
\itemsep=0pt
\setcounter{tocdepth}{1}
\tableofcontents
\newpage

\section{Introduction}
Set-valued predictors, unlike single-output ones, allow to provide a set of possible class candidates.
They became popular in recent years due to their ability to cope with class ambiguity, possibly present in multi-class problems.
This work focuses on set-valued methods for multi-class classification problem.
Unlike in single-output setup, it is not completely clear how to define \emph{the} set-valued classifier of interest due to various trade-offs that can be considered.
Therefore, one of our main goals is to provide an overview of various set-valued frameworks, previously considered in the literature, and complement them with their natural extensions.
We highlight similarities shared by all the formulations and emphasize on their differences.
In our exposition we follow a unified statistical setting where the description of the frameworks, methods, and underlying trade-offs become more transparent.

\subsection{From single-output to set-valued classifiers}

\begin{figure}
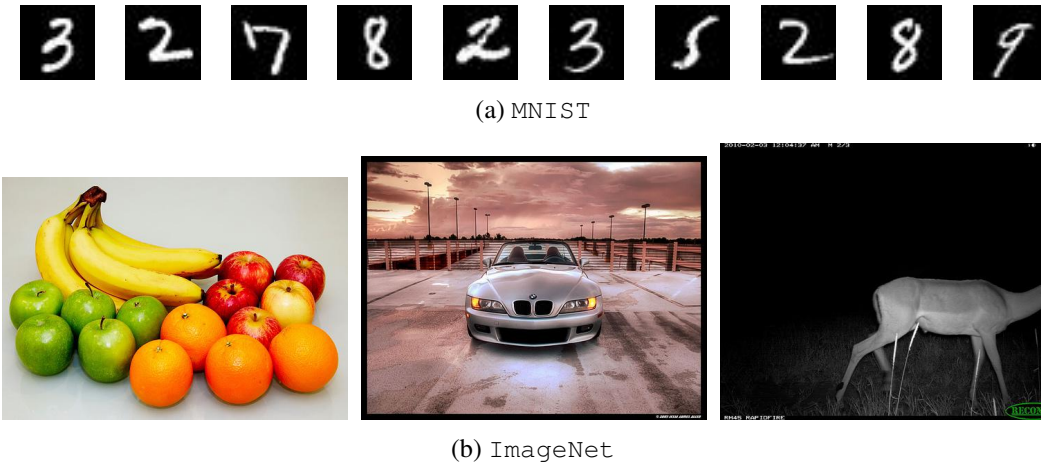

    \centering
    \begin{subfigure}{\textwidth}
        \hfill
        \foreach \i in {1,2,...,10} {
          \includegraphics{imgs/examples/mnist_ex\i}
          \hfill
        }
        \caption{\texttt{MNIST}}
        \label{fig:mnist-images}
    \end{subfigure}
    ~\\
    \begin{subfigure}{\textwidth}
        \hfill
        \foreach \i in {1,3,5} {
          \includegraphics[width=.32\textwidth]{imgs/examples/imagenet_ex\i}
          \hfill
        }
        \caption{\texttt{ImageNet}}
        \label{fig:imagenet-images}
    \end{subfigure}
    \caption{
        Examples from \texttt{MNIST} and \texttt{ImageNet} multi-class classification tasks.
    }
\end{figure}

Classically, multi-class classification considers the problem of learning a \emph{single-output classifier} $h : \class{X} \to [L] \eqdef \{1, \ldots, L\}$ based on data $(\bsX_1, Y_1), \ldots, (\bsX_n, Y_n) \in \class{X} \times [L]$ which minimizes the error rate $\Prob(Y \neq h(\bsX))$ on a new couple $(\bsX, Y)$.

For problems such as \texttt{MNIST} (\autoref{fig:mnist-images}), single-output classifiers have been very successful and we are now able to learn models with a very low error rate, \ie less than 1\% \citep{ciregan2012multi}.
However, nowadays, common datasets are much more complex leading to much higher errors, for instance the error on the \texttt{ImageNet} dataset (\autoref{fig:imagenet-images}) is around 20\% using state-of-the-art models \citep{xie2017aggregated}.
Such high error is not solely due to the difficulty to construct a good classifier, but rather due to the high ambiguity in the dataset.
As shown in \autoref{fig:imagenet-images}, some images are intrinsically ambiguous: in some images, several objects are present; in others, due to occlusion or noise in the image, it is not clear -- even to a human expert -- what class should be predicted.
Forcing a classifier to predict a single class will thus mechanically increase its error rate as in such cases several answers can be considered as correct.

A way to deal with this ambiguity is to allow the classifier to predict a set of candidate classes rather than a single one~\citep{grycko1993}.
These types of classifiers are called \emph{set-valued classifiers}.
They are defined as mapping from the input space $\class{X}$ to the set of all subsets of $[L]$, \ie the power set of $[L]$ denoted by $2^{[L]}$:
\begin{equation*}
    \Gamma: \class{X} \to 2^{[L]}\enspace.
\end{equation*}
Note that by considering set-valued classifiers we do not alter the underlying multi-class data-generating process, that is, the observed data $(\bsX_1, Y_1), \ldots, (\bsX_n, Y_n) \in \class{X} \times [L]$ is still that of multi-class -- each instance $\bsX_i$ is tagged by a unique class $Y_i$ in the sample. Instead, we change the form of prediction functions passing from single outputs to the set-valued ones.

The easiest way to build such a classifier is to always predict a fixed number of classes such as the top-5 most probable classes.
This approach is in fact the one chosen in the official \texttt{ImageNet} classification challenge \citep{imagenet}.
However, as shown in the previous images, in general, there is no reason to predict exactly 5 or any other a priori fixed number of classes all the time.
This observation generalizes to most of recent datasets.
In such cases, it is possible to provide a more informative prediction by relaxing this hard constraint on the set size.

In general, it is thus necessary to develop other formulations of set-valued classification strategies which could be more suited for the problem at hand.
Most of the known formulations can be seen as different ways to choose a trade-off between two complementary quantities: the \emph{error rate} and the \emph{set size} (see Section~\ref{subsec:unif_framework} for its formal definition).
The former quantifies the accuracy of set-valued classifier $\Gamma$ -- how likely $\Gamma(\bsX)$ contains the underlying truth $Y$; The latter quantifies its informativeness -- informative predictions $\Gamma(\bsX)$ should not contain too many candidates.
For instance, top-$5$ classification strategy consists in minimizing the average error rate $\Prob(Y \notin \Gamma(\bsX))$ under the constraint that the set size $|\Gamma(\bsx)|$ is less than or equals to $5$ for every input $\bsx \in \class{X}$.

We emphasize that there is no formulation which is uniformly better than the others or one which is ``incorrect''.
On the contrary, each formulation is complementary to one another: they are suited to different contexts and applications.
The goal of this work is too provide a concise overview of the literature on set-valued classification, compare different frameworks, and fill gaps that were not previously addressed.

\subsection{Set-valued classification -- a unified framework}
\label{subsec:unif_framework}

First of all recall the formal statistical setup of the multi-class classification problem. Let $(\bsX, Y)$ be a random couple taking values in $\class{X} \times [L]$ with joint distribution $\Prob$.
Here $\bsX$ represents the vector of features (\eg an image) and $Y$ is the unique class (\eg a digit present on the image).
An informal goal in the set-valued classification framework is to build a set-valued classifier $\Gamma : \class{X} \to 2^{[L]}$, which has two desired properties: its size $|\Gamma(\bsX)|$ is not too large and it is likely to contain the true class, \ie $Y \in \Gamma(\bsX)$. In order to construct a set-valued classifier with such properties, the practitioner has access to training data $(\bsX_1, Y_1), \ldots, (\bsX_n, Y_n) \simiid \Prob$, independent from $(\bsX, Y)$.

Formalization of the above intuitive idea, not only requires a rigorous definition of size and error of $\Gamma : \class{X} \to 2^{[K]}$, but also, in the spirit of decision theory, it requires a principled way to define an ``ultimate'' set-valued classifier $\Gamma^*$.
Unlike single-output prediction, mentioned in the previous section, where it is natural to target $g^* : \class{X} \to [K]$ which minimizes the 0/1-loss (accuracy), such an obviously natural choice does not exist in the set-valued classification framework.
As already mentioned, a good set-valued prediction always strikes for a trade-off between two quantities, a measure of the accuracy of the predicted sets and a measure of their informativeness.
In this paper, we will focus on a particular choices of such measures: respectively, {\it error rate}, and {\it set size}, which we now define.
\begin{itemize}
    \item {\bf Error rate}: this quantity can be defined point-wise or in average as
    \begin{align}
    \label{eq:Risk_Notions}
        &\error(\Gamma; \bsx) \eqdef \Prob(Y \notin \Gamma(\bsX) \mid \bsX = \bsx) &&\textbf{(point-wise error)}\enspace,\\
        &\error(\Gamma) \eqdef \Prob(Y \notin \Gamma(\bsX)), && \textbf{(average error)}\enspace.
    \end{align}
    Naturally, the average error is related to point-wise error by $\error(\Gamma) = \Exp_{\bsX}[\error(\Gamma; \bsX)]$.
    The error rate is associated to the accuracy of the set-valued classifier $\Gamma$. It is often named {\it coverage}, {\it recall}, or {\it risk}.
    \item {\bf Set size}: analogously set size can also be viewed in a point-wise manner or in average with the following definitions
    \begin{align}
    \label{eq:Size_Notions}
        &\size(\Gamma; \bsx) \eqdef |\Gamma(\bsx)| &&\textbf{(point-wise size)}\enspace,\\
        &\size(\Gamma) \eqdef \Exp_{\bsX} |\Gamma(\bsX)|, &&\textbf{(average size)}\enspace.
    \end{align}
    Similarly to the error, the size admit the same relation $\size(\Gamma) = \Exp_{\bsX}[\size(\Gamma; \bsX)]$.
    It is commonly accepted that the smaller the set size is the more informative the set-valued $\Gamma$ is. In the literature, the set size can also be called {\it expected set size}, {\it adaptive set size} or {\it information}.
\end{itemize}
The balance of the set size and the error rate can be achieved in various ways, and, depending on the application at hand should be considered using domain specific knowledge.
Nevertheless, most of the optimal set-valued classifiers of interest can be formalized via a minimization problem under, possibly, distribution dependent constraints.

\begin{equation}
\tag{{\bf$\class{P}$}}
\label{eq:general_definition_Gamma_star}
\begin{aligned}
        \Gamma^*_{(\objective, \constraint)} \in \argmin_{\Gamma : \class{X} \to 2^{[L]}} \; & \objective(\Gamma) \\ \nonumber
            \text{s.t. } & \Gamma \in \constraint
\end{aligned}\enspace,
\end{equation}
where $\objective(\cdot)$ is a real-valued functional on set-valued classifiers, which depends on the distribution; and $\constraint$ is a (distribution dependent) family of set-valued classifiers with desired properties.
Table~\ref{tab:names} provides a quick overview of the formulations that we describe in the paper. They correspond to particular choices of the functional $\objective(\cdot)$ and the family $\constraint$.
Within this unified framework, the optimal set-valued classifier $\Gamma^*_{(\objective, \constraint)}$ is considered as the gold standard. Consequently, the goal of the practitioner is to build a data-driven set-valued classifier $\hat{\Gamma}$ using the sample $(\bsX_1, Y_1), \ldots, (\bsX_n, Y_n)$ which preserves, as much as possible, all the desirable features (\eg error rate or size) of the optimal set-valued classifier $\Gamma^*_{(\objective, \constraint)}$.

Despite the large variety of pairs $(\objective, \constraint)$ that can be considered, most of the frameworks, previously described in the literature, admit a general plug-in driven approach for construction of $\hat\Gamma$.
It consists of two principal steps: first, derive the closed-form solution of the problem~\eqref{eq:general_definition_Gamma_star}; second, estimate all the unknown quantities in the expression for $ \Gamma^*_{(\objective, \constraint)}$.
\begin{table}[]
    \resizebox{\textwidth}{!}{%
        \centering
        \begin{tabular}{@{}llllll@{}}
    \toprule
    \textsc{Name} &
    \textsc{Objective} $\objective(\Gamma)$ &
    \textsc{Constraint} $\constraint$ &
    \textsc{Optimal} $\Gamma^*_{(\objective, \constraint)}$ &
    \textsc{Section} & \\
    \midrule
    Penalized &
    $\error(\Gamma) + \lambda \size(\Gamma)$ &
    N/A &
    $\Gamma^*_\lambda$: threshold &
    Section~\ref{sec:Soft_constraint_types} \\
    \fadedmidrule
    Point-wise size &
    $\error(\Gamma)$ &
    $\size(\Gamma; \bsx) \leq k$ &
    $\Gamma^*_k$: top-$k$ &
    Section~\ref{subsub:top-k} \\
    \fadedmidrule
    Average size &
    $\error(\Gamma$) &
    $\size(\Gamma) \leq \bar{k}$ &
    $\Gamma^*_{\bar{k}}$: threshold &
    Section~\ref{subsec:average_size} \\
    \fadedmidrule
    Point-wise error &
    $\size(\Gamma)$ &
    $\error(\Gamma; \bsx) \leq \epsilon$ &
    $\Gamma^*_\epsilon$: top-$k(\bsx)$ &
    Section~\ref{subsec:point_error} \\
    \fadedmidrule
    Average error &
    $\size(\Gamma)$ &
    $\error(\Gamma) \leq \bar{\epsilon}$ &
    $\Gamma^*_{\bar{\epsilon}}$: threshold &
    Section~\ref{subsec:average_error} \\
    \fadedmidrule
    F$_\beta$-score &
    $\frac{(1 + \beta^2)(1 - \error(\Gamma))}{\beta^2 + \size(\Gamma)}$ &
    N/A &
    $\Gamma^*_F$: threshold &
    Section~\ref{subsec:f-score} \\
    \fadedmidrule
    Hybrid size &
    $\error(\Gamma)$ &
    $
    \begin{aligned}
        & \size(\Gamma) \leq \bar{k}, \\
        & \size(\Gamma; \bsx) \leq k
    \end{aligned}
    $ &
    hybrid &
    Section~\ref{sec:hybrids}\\
    \fadedmidrule
    Hybrid error &
    $\size(\Gamma)$ &
    $
    \begin{aligned}
        & \error(\Gamma) \leq \bar{\epsilon}, \\
        & \error(\Gamma; \bsx) \leq \epsilon
    \end{aligned}
    $ &
    hybrid &
    Section~\ref{sec:hybrids} \\
    \bottomrule
\end{tabular}
    }
    \caption{
        Nomenclature and description. We recall the notation: average error rate $\error(\Gamma) = \Prob(Y \notin \Gamma(\bsX))$; point-wise error rate $\error(\Gamma, \bsx) = \Prob(Y \notin \Gamma(\bsX) \,|\, \bsX = \bsx)$; average set size $\size(\Gamma) = \Exp |\Gamma(\bsX)|$; point-wise set size $\size(\Gamma, \bsx) =  |\Gamma(\bsx)|$.
    }
    \label{tab:names}
\end{table}
The derivation of the explicit form of the optimal set-valued classifier $ \Gamma^*_{(\objective, \constraint)}$ assuming that the distribution $\Prob$ is known can be typically performed in analytic way.
The central objects are the marginal distribution of the features $\Prob_{\bsX}$ and the posterior distribution\footnote{Here $\Delta^{L - 1}$ stands for the probability simplex on $L$ atoms.} $\bsp = (p_1, \ldots, p_L)^\top : \class{X} \to \Delta^{L - 1}$ of the label $Y$ defined component-wise for all $\ell \in [L]$ as $p_{\ell}(\bsx) \eqdef \Prob(Y = \ell \mid \bsX = \bsx)$.
The former tells us how often an observation $\bsx$ (or its neighbourhood) can be observed, while the latter measures the relevance of label $\ell$ for the instance $\bsx$.
The explicit expression for the optimal set-valued classifier $ \Gamma^*_{(\objective, \constraint)}$ in all the considered frameworks reduces to one of the following three cases:
\begin{itemize}
    \item {\bf Thresholding}: there exists $\theta = \theta(\Prob, \constraint, \objective) \in [0, 1]$, such that for all $\bsx \in \class{X}$ the optimal set-valued classifier can be defined as
        \begin{align*}
            \Gamma^*_{(\objective, \constraint)}(\bsx) = \enscond{ \ell \in [L] }{ p_{\ell}(\bsx) \geq \theta } \enspace.
        \end{align*}
    In this case, for a given instance $\bsx$, the prediction $\Gamma^*_{(\objective, \constraint)}(\bsx)$ consists of those candidates whose probability $p_{\ell}(\bsx)$ surpasses a certain level $\theta$.\\
    {\bf NB.} $\theta = \theta(\Prob, \constraint, \objective)$ is allowed to depend on the distribution $\Prob$, which is unknown.
    \item {\bf Point-wise top-k}: for all $\bsx \in \class{X}$, there exists $k = k(\bsx, \Prob) \in [L]$, such that the optimal set-valued classifier can be defined as
    \begin{align*}
        \Gamma^*_{(\objective, \constraint)}(\bsx) = \Top_{\bsp}(\bsx, k)\enspace,
    \end{align*}
    where operator $\Top_{\bsp}(\bsx, k)$ outputs $k(\bsx, \Prob)$ labels with largest values of $\bsp(\bsx) = (p_1(\bsx), \ldots, p_L(\bsx))^\top$. Note that following this strategy implies that the point-wise size of $\Gamma$ satisfies $\size(\Gamma, \bsx) = k(\bsx, \Prob)$.\\
    {\bf NB.} $k = k(\bsx, \Prob)$ is allowed to depend on a point $\bsx \in \class{X}$ and the distribution $\Prob$.
    \item {\bf Hybrid}: the hybrid strategy can be obtained by a point-wise intersection of the thresholding rule and the point-wise top-$k$ rule. We describe several examples of such strategy in Section~\ref{sec:hybrids}. See also Table~\ref{tab:names} for some examples.\\
    {\bf NB.} we are not aware of previus contributions that consider hybrid frameworks.
\end{itemize}
We again emphasize that in case of thresholding, the value of $\theta = \theta(\Prob)$ can depend on the unknown distribution $\Prob$ and ought to be estimated.
Meanwhile, in case of point-wise top-$k$, the parameter $k = k(\bsx, \Prob)$ can depend on both: point $\bsx$ where prediction should be constructed as well as the unknown distribution.

\subsection{Organization of the paper}
Section~\ref{sec:PopularSvFramework} presents most popular set-valued classification strategies and illustrates them in simple examples.
People that are interested only on the numerical comparison of the different framework would directly go to Section~\ref{sec:numerical}. Hybrid methods, that involve sophisticate constraints or more that one constraint are introduced in Section~\ref{sec:hybrids}. In particular, we state in this section the closed-form of the optimal hybrid set-valued classifiers.

\subsection{Notation}
Throughout this work we use the following generic notation.
For every $L \in \bbN$, we denote by $[L] \eqdef \{1,\ldots,L\}$, the set of first $L$ positive integers.
For every finite set $A$, we denote by $|A|$ the cardinality of $A$.
The power set (the set of all subsets) of a finite set $A$ is denoted by $2^A$.
For a monotone decreasing function $F$ we denote by $F^{-1}$ it generalized inverse\footnote{See~\cite{embrechts2013note} for general definition of the generalized inverse of a monotone function.} defined point-wise as $F^{-1}(u) \eqdef \inf \enscond{ t \in \bbR }{ F(t) \leq u}$.

\section{Related classification frameworks}
\label{sec:related-classification-frameworks}
Set-valued classification has many connections with other classification setups that may lead to confusions and the differences may not well understood at the first sight.
The purpose of this section is to clarify the relation of set-valued classification with other frameworks and emphasize the key differences.

\subsection{Multi-label classification}
At the first glance set-valued classification seems to be directly connected to \emph{multi-label classification}~\citep{dembczynski2012label,zhang2013review}.
This is a recurrent source of confusion for readers and practitioners and we dedicate this section to clarify this point.
Mainly, the confusion comes from the fact that in both settings a set of label candidates is predicted as an output.
However, the crucial difference with set-valued classification lies in the data-generating process.
From purely practical perspective, if the classes $Y_i$'s in a given dataset are elements of $[L]$ and not of $2^{[L]}$, then there should be no confusion between the multi-class and the multi-label settings.
In other words, it is the training sample that dictates the setting and the underlying data generation process.

To be more precise, in multi-label classification each input $\bsx \in \class{X}$ is associated with a set of labels $\bsy \in \{0, 1\}^L$.
During training, we are given a set of such pairs $(\bsX_i, \bsY_i)_{i = 1}^n$ and the goal is to learn a \emph{multi-label classifier} $\bsh = (h^{(1)}, \ldots, h^{(L)})^\top : \class{X} \to \{0, 1\}^L$ from it.
A typical example of such data is the \texttt{Wikipedia} dataset~\citep{zubiaga2012enhancing,Bhatia16}, where $\bsX_i$ is a Wikipedia article and $\bsY_i$ describes the categories that are intended to group similar articles together.

Already at this level we can observe the difference between multi-labeled and multi-class data.
In the former,  a given input $\bsx$ is tagged by $\bsy = (y^{(1)}, \ldots, y^{(L)})^\top$, \ie $L$ distinct 0/1 entries.
Meanwhile, in the latter the label is a unique integer.
From the theoretical perspective, we note that in the multi-label setting each label $Y^{(\ell)}$ has a probability $q_\ell(\bsx) \eqdef \Prob(Y^{(\ell)} = 1 \mid \bsX = \bsx)$ to be present and $1-q_\ell(\bsx)$ to be absent from the set of labels for a given instance $\bsx$.
Note that in general the $\sum_{\ell = 1}^L q_\ell(\bsx)$ can, and most likely does, exceed one.
In contrast, in multi-class classification, the label $Y \in [L]$ of a given input $\bsx$ is sampled from the conditional probability $p_{\ell}(\bsx) = \Prob( Y=\ell \mid \bsX = \bsx )$ which is a categorical distribution, \ie $\sum_{\ell = 1}^L p_{\ell}(\bsx) = 1$.

Despite the above difference, the eventual goal in both frameworks is still to predict a set and any multi-class problem can be viewed as an instance of multi-label problem with additional dependency structure introduced on the classes.
From this point of view, any method tailored for multi-label problems can be used in the framework of set-valued classification.
Indeed, note that for any multi-label prediction $\bsh : \class{X} \to \{0, 1\}^L$ one can build a set-valued one $\Gamma_{\bsh}(\bsx) = \enscond{\ell \in [L]}{\bsh^{(\ell)}(\bsx) = 1}$.
However, since these methods are not constructed with the multi-class generation process in mind, their applicability is limited and set-valued minded algorithms are better suited in this case.

Of course there are other, more subtle, differences between the two frameworks. For instance, we did not discuss the choice of the risk measure for the multi-label classification nor we talked about possible dependencies inside $\bsY$.
These directions are exciting and bring a lot of interesting questions, but as it should be already clear, in a completely different framework from what we consider in this manuscript.

\subsection{Classification with reject option}

An alternative approach to address the ambiguity in multi-class classification problem is via the so-called classification with reject option~\citep{Chow57,Chow70,Herbei_Wegkamp06,Ni19,ramaswamy2018consistent,Wu07,Zhang18,denis2020consistency}.
In this framework, instead of predicting a set capturing the ambiguous classes, the classifier is allowed to abstain from prediction, that is, it informally says \emph{``I do not know''}.
This refusal has a cost -- a parameter of the problem in this setting -- depending on which lower error rate can be achieved compared to a single class outputs.

Speaking formally, the predictor with reject option is a function of the form $h: \class{X} \to [L] \cup \{ \circledR \}$ where $\circledR$ represents a refusal to answer.
One can view this predictor as a special form of set-valued classifier. Indeed, whenever $h$ outputs a value in $[L]$, then the corresponding set-valued classifier outputs a singleton.
Meanwhile, the rejection symbol $\circledR$ for $h$ corresponds to two situations for a set-valued classifier, either the set-valued classifier outputs the empty set or more than one label.
Consequently, the classifiers with rejection are strictly included in the set-valued predictors, since the latter allow a more flexible quantification of ``I do not know'' prediction.
%

\subsection{Conformal prediction}
\label{sec:ConformalPrediction}
Seminal ideas of set-valued classification appear also in the works by Vovk, Gammerman, and Shafer on conformal prediction theory~\citep{vovk2005algorithmic}.
Given a set of supervised data $(\bsX_1, Y_1), \ldots, (\bsX_n, Y_n)$, a new instance $\bsX$, and a confidence level $\alpha$ the goal in conformal prediction theory is to provide a set $\Gamma_{\alpha}(\bsX) = \Gamma_{\alpha}((\bsX_1, Y_1), \ldots, (\bsX_n, Y_n), \bsX) \subset 2^{[L]}$, which satisfies
\begin{align*}
    \mathbf{P}(Y \notin \Gamma_{\alpha}(\bsX)) \leq \alpha\enspace,
\end{align*}
where $\mathbf{P}$ stands for the joint distribution of $(\bsX_1, Y_1), \ldots, (\bsX_n, Y_n), (\bsX, Y)$.
To achieve this goal, the conformal prediction theory relies on the exchangeability assumptions, which is slightly more general than the \iid assumption.
Moreover, conformal theory relies on the conformity measure which assigns a level of similarity of each instance $(\bsX, Y)$ \wrt the data.

Though we do not focus on the conformal predictions, the reader should keep in mind that this powerful framework should not be put against the direction that we take here. Rather, it should be viewed as a viable alternative (and sometimes complementary) way to build set-valued classifiers, with strong expected error rate or expected size guarantees, having its advantages and disadvantages over the framework described here.
The major advantage is of course the assumption free average error rate (or average size) guarantees, while possible disadvantages include high computational complexity and randomized nature of the constructed classifiers. The latter means that for the same input $\bsX$ the set $\Gamma_{\alpha}(\bsX)$ can be different from one experiment to another depending on the underlying randomness of $\Gamma_{\alpha}$.
For a broad review of conformal prediction theory with main theoretical and practical advances we refer to~\cite{vovk2005algorithmic, vovk2017}.


\section{Practical motivations and applications}
\label{sec:practical-applications}
In general, set-valued classification can be applied on top of any multi-class problem.
At the same time, in certain scenarios, the set-valued approach might be strictly preferred to the single-output paradigm.
This is especially the case in applications where the observed data $(\bsX_i, Y_i)$ comes from a multi-class classification problem -- we observe a unique class $Y_i \in [L]$ associated with the observation $\bsX_i$ -- but, the uniqueness of the class is not actually intrinsic to the problem at hand.
In this section we describe several applications and data acquisition processes where set-valued classification paradigm might be preferred.


\subsection{Fault-intolerant applications}
Set-valued classification framework arises naturally in {\it fault-tolerance} type problems. The goal is to produce a prediction with (almost) null error.
While such guarantee can hardly be expected from a single output predictor, accurate as it may be, drastically low errors can be achieved by set-valued classifiers.
This is especially the case if the error rate is added as a desirable constraint (\emph{cf.} Section~\ref{subsec:average_error}). Typical examples of this problem are connected with medical diagnosis or aeronautic purposes -- areas where the price for a wrong prediction is too high to be tolerated~\citep{kourou2015machine,lambin2017decision}. In these type of applications the set-valued classification framework can serve as a replacement of a more standard approach based on re-weighting~\citep{turney1994cost,elkan2001foundations} of errors and modification of score function~\citep{ling2003auc}.

\subsection{Highly ambiguous datasets}
High ambiguity between classes emerges in numerous applications and this is the bedrock of the use of set-valued frameworks.
In practice, it is very common that modern large-scale datasets contain a fair amount of ambiguity among classes.
It is especially the case in standard Fine-Grained Visual Categorization (FGVC) datasets such as \texttt{PlantCLEF2015}~\citep{plantclef2015}, \texttt{ImageNet}~\citep{imagenet}\footnote{See \url{https://sites.google.com/view/fgvc7} for additional examples.}.


As an example, consider a classification scenario where the output $y_i$ lies in $[L]$.
At the same time, the input $\bsx_i$ is a representation of a latent object $\bso_i$ which is indeed associated with a single label $y_i \in [L]$.
However, the observed input $\bsx_i$ is a blurred, noisy, corrupted, or partial version of $\bso_i$ and it may be hard to recover the true label $y_i$ based solely on $\bsx_i$.
Notably, the noise in measurement is actually responsible for the ambiguity and set-valued classifiers can serve as a remedy in this case.

A prominent example where such processes occur is connected to plant image recognition~\citep{plantclef2015,champ2015comparative,champ2016categorizing}.
Consider the case where the latent object $\bso_i$ is a given plant species and $\bsx_i$ is its image typically taken by a smartphone.
While a botanist expert of this flora is able to provide the exact specie $y_i$ of the object $\bso_i$, it is possible that even an expert may fail to classify based on the image $\bsx_i$.
Indeed, the latter might be too blurred or might consist of only one organ of the plant, say the flower, which may not provide enough information to be perfectly discriminated from other species.
On the other hand, based on the partial information provided by $\bsx_i$, the expert is typically able to extract a small subset of species $\Gamma(\bsx_i) \subset [L]$ that is likely to contain the true description $y_i$ of $\bso_i$. In this case, the expert naturally acts as if he is a set-valued classifier.



\subsection{Collecting a single positive label in multi-label data}

\begin{figure}[t!]
    \centering
    \begin{subfigure}[b]{.3\textwidth}
        \includegraphics[width=\textwidth]{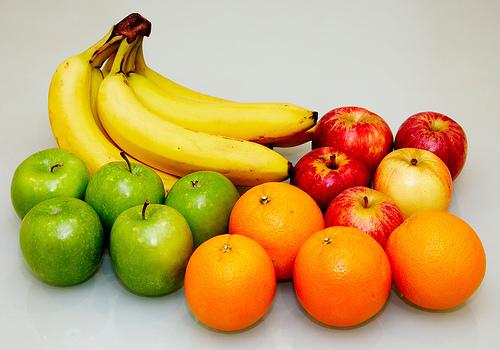}
        \caption{\textbf{banana}, orange, Granny Smith}
        \label{fig:imagenet-multiobject-example}
    \end{subfigure}
    \hfill
    \begin{subfigure}[b]{.3\textwidth}
        \includegraphics[width=\textwidth]{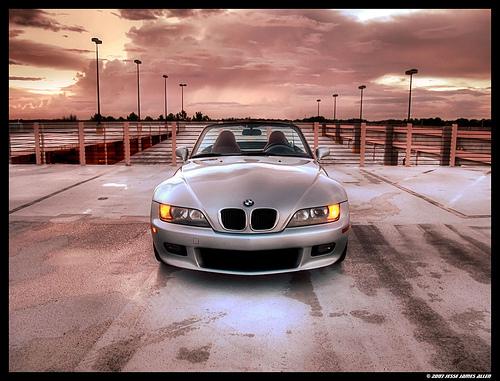}
        \caption{\textbf{sports car}, convertible}
        \label{fig:imagenet-multiattribute-example}
    \end{subfigure}
    \hfill
    \begin{subfigure}[b]{.3\textwidth}
        \includegraphics[width=\textwidth]{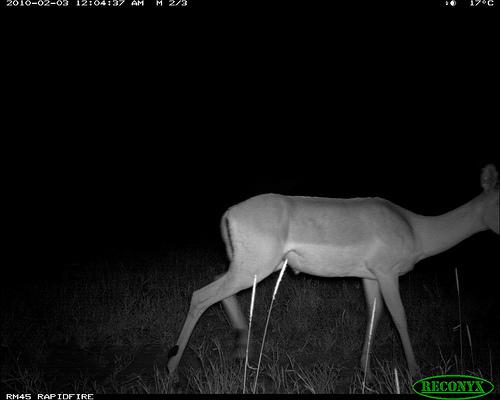}
        \caption{\textbf{impala}, gazelle, hartebeest}
    \end{subfigure}
    \caption{
        Examples from \texttt{ImageNet} containing either multiple objects (a), or a single object with multiple matching attributes (b), or an intrinsically ambiguous image containing a single object associated with a single class (c).
        The official annotation associated is shown in bold.
    }
\end{figure}

Unlike highly ambiguous datasets, positive-only data is inherently multi-label: each observation contains several objects or attributes belonging to different classes, see \Cref{fig:imagenet-multiobject-example,fig:imagenet-multiattribute-example}.
However, in certain cases, annotating multi-label data can be very expensive since a positive or negative label must be collected for every class.
A common modeling assumptions was described in~\citep{royle2012likelihood,hastie2013inference,mac2019presence} where only a \emph{single} positive label $y_i \in [L]$ is reported for each instance $\bsx_i$.
The popularity of this framework stems from the fact that it is typically easy to detect a single present object on a picture, while it is increasingly harder and more time consuming to detect all the present objects.
Note how in this case, the inherently multi-label problem is transformed into a multi-class one due to the data gathering process.
As an example, one might imagine the scenario where a human expert observes an image $\bsx_i$ which might involve several objects, and reports only the first one that was noticed ignoring the rest.
The most famous benchmark dataset which falls within the scope of this framework is \texttt{ImageNet}~\citep{imagenet}, where each image $\bsx_i$ is annotated by one and only one class $y_i \in [L]$.
 \section{Main set-valued classification frameworks}
 \label{sec:PopularSvFramework}
Several set-valued classifiers fall within the formulation~\eqref{eq:general_definition_Gamma_star}. Because there have been contributions from different fields and there is no standardized naming for the different formulations, providing an exhaustive overview of the literature is impossible. Therefore, in this section we made the choice to focus on the four most popular and intuitive set-valued classification frameworks for multi-class setting described by such unified formulation~\eqref{eq:general_definition_Gamma_star}. (We differ to Section~\ref{sec:discussion_other_framework} a discussion on more involved frameworks which also fall with this formulation or on other formulations obtained as hybrids of the four main formulations since they are not considered in the comparison study of this work.) In these four formulations, the set $\constraint$ in \eqref{eq:general_definition_Gamma_star} constrains whether the error rate or the size of the set-valued classifier.

This section is organized in four paragraphs, each of them is devoted to one set-valued formulation. For each particular set-valued problem, we define the set-valued classifier as a particular solution of the formulation~\eqref{eq:general_definition_Gamma_star}, exhibit a closed-form of an optimal solution, illustrate the strategy to build a data-driven procedure, discuss its advantages and disadvantages, and provide a brief literature review.

Before proceeding, let us comment on some general considerations (that can be skipped by the reader without any impact): the optimal solutions that we will provide highlight a general plug-in type approach to construct practically relevant set-valued classification procedures for all frameworks. In particular, all the optimal set-valued classifiers rely on the conditional probabilities $\bsp = (p_1, \ldots, p_L)^\top$.
Therefore, in order to build a data-driven counterpart, we first need an estimator $\hat\bsp = (\hat p_1, \ldots, \hat p_L)^\top$ of $\bsp $, and this step requires a \emph{labeled} dataset. Furthermore, in order to enforce the constraint in~\eqref{eq:general_definition_Gamma_star}, it turns out that the estimation step for some set-valued classifiers is more involved and then needs the estimation of an additional parameter that depends on the underlying data-distribution.
As a consequence, the estimation procedures can be categorized into three types depending on whether they require \emph{no additional data}, an additional \emph{unlabeled} data, or an additional \emph{labeled} additional data.
Note that in the case where additional unlabeled data is required, if such data was already available beforehand, \ie in the semi-supervised setting~\citep{Vapnik98}, then it can be leveraged directly by the estimation procedure. However, as we will see, in most cases, we do not require a lot of additional data, less than for training $\hat\bsp$, unlike most semi-supervised cases where unlabeled data is abundant, more than training data. If, on the other hand, only labeled data is available, then the estimation procedure might require to split this data is two and to discard some labels.



Lastly, let us emphasize that for some of the results of this section we need the following mild assumption.
\begin{assumption}
    \label{ass:continuity_global}
    Random variables $p_{\ell}(\bsX)$ are non-atomic for all $\ell \in [L]$.
\end{assumption}
This assumption is equivalent to assuming that the distribution of $p_{\ell}(\bsX)$ is continuous for all $\ell \in [L]$.
Essentially, Assumption~\ref{ass:continuity_global} is a sufficient condition under which several set-valued frameworks are well defined.
That is to say that the optimal set-valued classifier $\Gamma^*_{(\objective, \constraint)}$ exists, is unique, and is deterministic.
We refer to~\citep{chzhen2019minimax,Sadinle2019} for a broad discussion on the role of Assumption~\ref{ass:continuity_global}.
In particular,~\citep{Sadinle2019} describe a way to bypass this assumption by considering randomized set-valued classifiers.
As it is not the main scope of this paper we resort to Assumption~\ref{ass:continuity_global} when necessary to avoid this technicality.

\subsection{Point-wise control -- almost sure type constraints}
In this part we describe those frameworks that can be formulated withing the unified paradigm~\eqref{eq:general_definition_Gamma_star} with a set $\constraint$ defined by a point-wise type constraint on the classifier $\Gamma$.
\subsubsection{Point-wise size control: Top-$k$ classifier.}
\label{subsub:top-k}

In classification tasks, it is common to replace the top-1 error rate by a top-k error rate.
For instance, \texttt{ImageNet} classification task \citep{imagenet} uses top-5 error rate as the main performance measure to account for the possible presence of several objects of different classes in the images.

This first natural approach in set-valued prediction simply consists in predicting sets of same size $k$ for every input:
\begin{align*}
    |\Gamma(\bsx)| = k\enspace,
\end{align*}
which is the most straightforward generalization of the single-output strategy.
In this case, the problem can be formulated as finding the optimal classifier $\Gamma^*_k$ solving the following problem,
\usetagform{bold}
\begin{equation} \tag{point-wise size control}
\label{eq:top_k}
\begin{aligned}
    \Gamma^*_k \in \argmin_\Gamma \; & \Prob(Y \notin \Gamma(\bsX) ) \\
    \text{s.t. } &  |\Gamma(\bsx)| \leq k \quad\text{a.s. } \Prob_{\bsX}
\end{aligned}
\end{equation}
\usetagform{default}
Note that for this formulation $\constraint = \enscond{\Gamma}{|\Gamma(\bsx)| \leq k \text{ a.s. } \Prob_{\bsX}}$. Though, formally this definition depends on the marginal distribution $\Prob_{\bsX}$ one can immediately notice that the constraint ${|\Gamma(\bsx)| \leq k \text{ a.s. } \Prob_{\bsX}}$ can be replaced by $|\Gamma(\bsx)| \leq k$ for all $\bsx \in \class{X}$.
The attractive feature of such a formulation is that the optimal classifier $\Gamma^*_k$ admits a simple and intuitive closed-form solution -- at every point $\bsx$ it outputs $k$ most probable candidates.
\begin{lemma}
\label{lem:top_k_oracle}
For all $\bsx \in \class{X}$ the classifier $\Gamma^*_k$ can be obtained by
\begin{align*}
    \Gamma^*_k(\bsx) = \Top_{\bsp}(\bsx, k)\enspace,
\end{align*}
where for all $\bsx \in \mathcal{X}$, and $k \in [L]$, the operator $\Top_{\bsp}(\bsx, k)$ outputs the $k$ labels with largest values among $p_1(\bsx), \ldots, p_{L}(\bsx)$.
\end{lemma}

Top-$k$ classifier is the most straightforward approach to set-valued classification and provides an exact control on the point-wise set size.
This simplicity is at the price of a lack of adaptation to the heterogeneity of the problem.
Indeed, top-$k$ outputs $k$ labels for every $\bsx \in \X$ independently of the level of ambiguity of the samples.
Furthermore, the exact value of $k$ to use might not always be clear beforehand and might require a posteriori refinement.
Finally, this framework does not provide any control over the error rate.

\myparagraph{Estimation}
To estimate this set-valued classifier, a single labeled training dataset is required, no additional data is necessary.
In particular, the plug-in top-$k$ classifier is defined element-wise by
\begin{align*}
  \forall \bsx \in \X, \quad \hat{\Gamma}_k(\bsx) = \Top_{\hat{\bsp}}(\bsx, k) .
\end{align*}
Note that in fact this estimator does not require to estimate exactly the conditional probability $\bsp(\bsx)$ as any scoring function preserving the ranking of $\bsp(\bsx)$ will give exactly the same classifier.


\myparagraph{Bibliographic references}
This formulation has long ago been used to compare models in practice. \citet{lapin2015top} address the~\myref{eq:top_k} formulation by providing a surrogate hinge loss. In particular, they illustrated the superiority of their modified top-$k$ procedure as compared to One-Versus-All SVMs.
Several works followed dealing with convex surrogates for this problem.
\citet{Lapin2016} proposed a new hinge surrogate of the $0/1$-loss as well as a modified cross-entropy loss, and, following the classical results of~\citet{zhang2004statistical, Bartlett_Jordan_McAuliffe06}, they defined the notion of ``top-$k$ calibration'', a property that is satisfied by their introduced loss functions, see also~\citet{Yang2020}.
\citet{Berrada2018} modified the loss proposed by \citet{lapin2015top} in order to provide a smooth loss function better adapted for neural networks and showed that the resulting method achieves state-of-the-art performance on classical image datasets.


From theoretical point of view~\cite{Titouan20} proved a bound on the error rate of the top-$k$ set-valued classifier based on the $\ell_1$-estimation error of the conditional probabilities.

\subsubsection{Point-wise error control:}
\label{subsec:point_error}
Another formulation consists in minimizing the average set size under a point-wise error rate constraint
\usetagform{bold}
\begin{equation}
    \label{eq:point_wise_error}
    \tag{point-wise error control}
    \begin{aligned}
        \Gamma^*_\epsilon \in \argmin_\Gamma \; &  \Exp_{\bsX} |\Gamma(\bsX)| \\
        \text{s.t. } & \Prob(Y \notin \Gamma(\bsX) \mid \bsX = \bsx) \leq \epsilon,  \text{ a.s. } \Prob_{\bsX}
    \end{aligned}
\end{equation}
\usetagform{default}
We remark that in this case $\constraint = \enscond{\Gamma}{\Prob(Y \notin \Gamma(\bsX) \mid \bsX = \bsx) \leq \epsilon, \text{ a.s. } \Prob_{\bsX}}$. Again, as in the previous paragraph the dependency on $\Prob_{\bsX}$ can be dropped.

The closed-form expression of $\Gamma^*_\epsilon$ is again rather intuitive. It is closely connected to the top-$k$ strategy, however, unlike previous section, where $k$ was fixed beforehand, the size of the output set can vary from one point to another in this case.
The classifier $\Gamma^*_\epsilon$ can be supported by the following intuition: ``Given point $\bsx$ the prediction $\Gamma^*_\epsilon(\bsx)$ consists of $k_{\epsilon}(\bsx)$ most probable classes, such that the cumulative probability of $(L - k_{\epsilon}(\bsx))$ not included classes is at most $\epsilon$''.

\begin{lemma}
    \label{lem:hard_coverage_control_oracle}
    For every $\epsilon$ an optimal set-valued classifier $\Gamma^*_\epsilon$ can be obtained for all $\bsx \in \mathcal{X}$ as
    \begin{align*}
      \Gamma^*_\epsilon(\bsx) = \Top_{\bsp}(\bsx, k_\epsilon)\enspace,
    \end{align*}
    where $k_\epsilon = k_\epsilon(\bsx)$ is such that
    \begin{align*}
        k_\epsilon(\bsx) = \min \enscond{k \in [L]}{\sum_{\ell = 1}^{k} p_{(\ell)}(\bsx) \geq 1 - \epsilon} \enspace.
    \end{align*}
\end{lemma}
\begin{proof}
    Postponed to appendix
\end{proof}

Theoretically, this formulation provides strong guarantees on the point-wise error rate.
It is thus very relevant for scenarios where the error rate must be controlled precisely.
One major drawback of this formulation is that it does not provide any control over the set size and can thus predict large sets which can then be uninformative.
A practical limitation is that it is hard to measure the point-wise error rate and thus the theoretical guarantees are often violated in practice.

\myparagraph{Estimation}
To estimate this set-valued classifier, a single labeled training dataset is required, no additional data is necessary.
In particular, the plug-in set-valued classifier is defined element-wise by
\begin{align*}
  \forall \bsx \in \X, \quad \hat{\Gamma}_{\epsilon}(\bsx) = \Top_{\hat{\bsp}}(\bsx, \hat{k}_\epsilon(\bsx)) .
\end{align*}
with
\begin{equation*}
  \hat{k}_{\epsilon} (\bsx) = \min \enscond{ k \in [L] }{ \sum_{\ell = 1}^{k} {\hat{p}}_{(\ell)}(\bsx) \geq 1 - \epsilon }
\end{equation*}
where  ${\hat{p}}_{(\ell)}(\bsx)$ is a reordering of $\hat{p}_{1}(\bsx), \ldots, \hat p_L(\bsx)$ in decreasing order, \ie $\hat{p}_{(1)}(\bsx) \geq \ldots \geq {\hat{p}}_{(L)}(\bsx)$.
Note that this estimator requires to estimate exactly the conditional probability $\bsp(\bsx)$ as they serve to estimate the point-wise error rate on which we set our constraint.
The model must thus output calibrated probabilities which requires additional care during training or an extra post-processing step.

Note that due to the finite sample effects the point-wise error rate might actually exceed the desired level $\epsilon$.
If such a behaviour is undesirable, a simple practical remedy can be provided via the off-set strategy. This approach is exactly the same as described above, but with
\begin{equation}
\label{eq:correction_pointwise_error}
  \hat{k}_{\epsilon} (\bsx) = \min \enscond{ k \in [L] }{ \sum_{\ell = 1}^{k} {\hat{p}}_{(\ell)}(\bsx) \geq 1 - \epsilon + r_{n, L} }\enspace,
\end{equation}
where $r_{n, L}$ is some positive sequence of the size of the dataset and the number of classes. Intuitively, this strategy requires to pick a smaller value of $\epsilon$ for the estimator.
We illustrate this strategy on Figure~\ref{fig:c_pointwise-error-rate-constraint-estimation} with $r_{n, L} = \sqrt{\sfrac{L}{n}}$.

\myparagraph{Bibliographic references}
Until very recently, most references dealing with set-valued classifiers that satisfy the point-wise error rate constraint (most of the time referred to as {\it conditional validity}) is related to the regression setting. For instance,
\citep{Cai2014Confidence,Lei_Wasserman2014,LeiGsellRinaldoTibshiraniWasserman_16} proposed an asymptotic study of predictors based on the point-wise error rate constraint under smoothness condition on the regression function. In addition, \citet{barber2019limits} studied a relaxed version of the point-wise error rate requirement.
More recently~\citep{Gyorfi_Walk_20} established finite sample distribution-free control on the tail distribution on the point-wise error rate and on the size of the predictor. The latter reference also provided strong consistency results under smoothness conditions on the regression function.

One classical result when we deal with finite sample control on the point-wise error rate is due to~\citet[Lemma~1]{Lei_Wasserman2014} (extended to the classification setting by Vovk~\citep{vovk_13}) that establishes that no control on the point-wise error rate can hold (for continuous marginal $\Prob_{\bsX}$) in a distribution-free setting unless with trivial set-valued predictors.

According to the classification setting, Vovk~\citep{vovk_13} provided a conformal predictor type algorithm to ensure the point-wise error rate validity on the training sample.
Recently~\citep{Gyorfi_Walk_20}, proposed a $k$NN based set-valued classifier and derived, under smoothness conditions on the $p_{\ell}$'s, finite sample bounds on the conditional error rate and on the point-wise size of the provided set. About the same time~\citep{Romano_Sesia_Candes20} developed a method that also attempt to get approximate conditional error rate guarantee.


\subsection{Average control -- constraints in expectation}
\label{sec:Soft_constraint_types}
Unlike standard top-$k$ predictions, general set-valued classification framework allows varying sizes of the set-valued classifier $\Gamma$, depending on the ambiguity in the distribution.
A natural methodological approach to build suitable set-valued classifiers lies in the minimization of the penalized risk that takes into account both risk and size.
That is, for $ \lambda >0$, we set
\usetagform{bold}
\begin{align}
\label{eq:penalized}
\tag{penalized version}
    \Gamma^*_{ \lambda } \in \argmin_\Gamma \; & \Prob(Y \notin \Gamma(\bsX) ) +  \lambda \  {\Exp_{\bsX}}|\Gamma(\bsX)|\enspace.
\end{align}
\usetagform{default}
This framework is among the first to be proposed and analyzed in the literature by \citet{Ha1996,Ha1997,Ha1997b}.
It is also known as \emph{class-selective rejection} \citep{Ha1996} or \emph{class-selection} \citep{LeCapitaine2014}.
It can be seen as a generalization of classification with reject option to the multi-class setting~\cite{Herbei_Wegkamp06}.
To the best of our knowledge, from a statistical learning perspective, this framework has not been much studied.
In particular, we have not found any study of empirical risk minimization procedures, nor proposal of surrogate losses for this risk.
Notably, the set-valued prediction $\Gamma^*_{\lambda}$ admits a closed-form solution written in a simple thresholding form.
\begin{lemma}
    \label{lem:penalized_oracle}
    For all $\lambda \geq 0$ and all $\bsx \in \class{X}$ it holds that
    \begin{align*}
        \Gamma^*_{\lambda}(\bsx) = \enscond{\ell \in [L]}{p_{\ell}(\bsx) \geq \lambda}\enspace.
    \end{align*}
\end{lemma}

When the cost/threshold $\lambda$ is fixed, this formulation is very simple.
However, choosing the appropriate value for $\lambda$ beforehand can be hard in practice.
In particular, the way this parameter controls the error rate and the set size is not explicit.
To overcome this difficulty, the next formulations control these quantities explicitly.

This framework also admits a simple estimation strategy.
In particular, the plug-in set-valued classifier is defined element-wise by
\begin{align*}
  \forall \bsx \in \X, \quad \hat{\Gamma}_{\lambda}(\bsx) = \enscond{\ell \in [K]}{ \hat{p}_\ell(\bsx) \geq \lambda} .
\end{align*}
Also, in general, as this formulation is expressed as an unconstrained risk minimization problem, the decision rule can be learned directly using empirical risk minimization approaches removing the need to estimate the conditional probability explicitly, much alike traditional classification methods.

\subsubsection{Average set size control:}
\label{subsec:average_size}
Instead of the above penalized version of the problem, one can consider its constrained counterparts, for instance, minimizing the error given a constraint on the expected set size
\usetagform{bold}
\begin{equation}
    \label{eq:average_size}
    \tag{average size control}
    \begin{aligned}
    \Gamma^*_{\bar{k}} \in \argmin_\Gamma \Prob(Y \notin \Gamma(\bsX) ) \\
    \phantom{\Gamma^*_{\bar{k}} \in \argmin_\Gamma}\text{s.t. } \Exp_{\bsX} |\Gamma(\bsX)| \leq \bar{k} \enspace.
    \end{aligned}
\end{equation}
\usetagform{default}
For this framework the set $\constraint = \enscond{\Gamma}{\Exp_{\bsX} |\Gamma(\bsX)| \leq \bar{k}}$.
Note how in this case, unlike previous constrained formulations, the set $\constraint$ unavoidably depends on the marginal distribution $\Prob_{\bsX}$ of $\bsX$. In particular, it is, in general, impossible to say a priori if a given set-valued classifier $\Gamma$ belongs to this set.

It is common that the constrained versions of optimization problems are in some sense equivalent to its penalized counterparts.
This phenomena occurs here as well following the next result of~\citet{Denis2017}.
\begin{lemma}
Let Assumption~\ref{ass:continuity_global} be satisfied. Fix $\bar{k} \in (0, K)$ and define
\begin{align}
    \label{eq:G_func}
    G(t) = \sum_{\ell = 1}^L \Prob(p_{\ell}(\bsX) \geq t)\enspace.
\end{align}
Then, an optimal set-valued classifier $\Gamma^*_{\bar{k}}$ can be obtained for all $\bsx \in \class{X}$ as
\begin{align*}
    \Gamma_{\bar{k}}^*(\bsx) = \enscond{\ell \in [L]}{p_{\ell}(\bsx) \geq G^{-1}(\bar{k})}\enspace,
\end{align*}
where $G^{-1}(\cdot)$ is the generalized inverse of $G(\cdot)$.
\end{lemma}
Note that if we pick the penalization parameter $\lambda = G^{-1}(\bar{k})$, then the~\myref{eq:penalized} yields the same optimal set-valued classifier as the one resulting from the present formulation. Yet, the choice $\lambda = G^{-1}(\bar{k})$ cannot be made without additional data.

This formulation can be seen as an adaptive version of the \myref{eq:top_k} formulation which is able handle the heterogeneity of the task ambiguity.
As a consequence, the present optimal set-valued classifier yields a smaller error rate than that of top-$k$ classifier (since top-$k$ is feasible for the \myref{eq:average_size} formulation with $\bar k = k$).
On the other hand, the threshold $G^{-1}(\bar{k})$ is distribution dependent and thus relies on the marginal distribution $\Pr_{\bsX}$.
It should be noted that this formulation does not provide an explicit control over the error rate which can be potentially large for individual samples.

\myparagraph{Estimation}
To estimate this set-valued classifier, an additional unlabeled dataset can be used to estimate the threshold $G^{-1}(\bar k)$.
Indeed, the constraint involved in this set-valued classifier relies on the marginal distribution $\Pr_{\bsX}$, and, in principle, unlabeled data would allow to estimate this constraint.
Denote by $\bsX'_1, \ldots, \bsX'_N$ the unlabeled dataset of size $N$ sampled independently from $\Pr_{\bsX}$ and define the following function
\begin{align}
    \label{eq:estimationCDFAverageSize}
  \forall t \in [0,1], \quad
  \hat{G}(t) = \frac{1}{N} \sum_{i=1}^{N} \sum_{\ell = 1}^L \ind{ \hat{p}_\ell(\bsX'_i) \geq t} .
\end{align}
This function is essentially the empirical counterpart of the function $G$ from Eq.~\eqref{eq:G_func} computed using the unlabeled dataset and for the estimator $\hat{\bsp}$.
The plug-in set-valued classifier is then defined element-wise by
\begin{align*}
  \forall \bsx \in \X, \quad
  \hat{\Gamma}_{\bar{k}}(\bsx) = \enscond{\ell \in [L]}{ \hat{p}_\ell(\bsx) \geq \hat{G}^{-1}(\bar{k}) } ,
\end{align*}
where $\hat{G}^{-1}$ is the generalized inverse of $\hat{G}$.

\myparagraph{Bibliographic references}
The present framework was introduced by \citet{Denis2017}, where the authors proposed a semi-supervised procedure based on empirical risk minimization. They derived rates of convergence under smoothness conditions on the conditional probabilities.
The idea to use unlabeled data to build set-valued classifiers was first discovered and developed by~\citet{Denis2017}, where the authors proposed a two step empirical risk minimization procedure and derived rates of convergence.
Later, \citet{Chzhen2019} developed a minimax analysis of this framework and derived optimal rates of convergence for a semi-supervised approach based on plug-in under smoothness conditions.
They showed the superiority of semi-supervised approaches over their supervised counterparts in certain situations.

\subsubsection{Average coverage control:}
\label{subsec:average_error}
The previous approach attempt to find the best set-valued classifier among those that have a desired average size.
Yet, doing so does not actually give any guarantees on the actual coverage, apart from being the smallest in the set of classifiers of interest.
Alternatively, we can target the minimization of the expected set size given constraint on the error rate.
This problem can be formulated as follow
\usetagform{bold}
\begin{equation*}
    \label{eq:average_error}
    \tag{average error control}
    \begin{aligned}
        \Gamma^*_{\bar{\epsilon}} \in \argmin_\Gamma \; &  \Exp_{\bsX} |\Gamma(\bsX)| \\
        \text{s.t. } & \Prob(Y \notin \Gamma(\bsX)) \leq \bar{\epsilon}
        \enspace .
    \end{aligned}
\end{equation*}
\usetagform{default}
For this framework the set $\constraint = \enscond{\Gamma}{\Prob(Y \notin \Gamma(\bsX)) \leq \bar{\epsilon}}$ and it unavoidably depends on the whole joint distribution $\Prob$ of $(\bsX, Y)$.

Besides, again, this constrained formulation is closely tied to the penalized version, thanks to the following result of~\citet{sadinle2019least} which characterizes the optimal set-valued classifier $\Gamma^*_{\bar{\epsilon}}$.
\begin{lemma}
Let Assumption~\ref{ass:continuity_global} be satisfied. Fix $\bar{\epsilon} \in (0, 1)$ and define
\begin{align}
    \label{eq:H_func}
    H(t) = \Prob(p_{Y}(\bsX) \geq t)\enspace.
\end{align}
Then, an optimal set-valued classifier $\Gamma^*_{\bar{\epsilon}}$ can be obtained for all $\bsx \in \class{X}$ as
\begin{align*}
    \Gamma_{\bar{\epsilon}}^*(\bsx) = \enscond{\ell \in [L]}{p_{\ell}(\bsx) \geq H^{-1}(1 - {\bar{\epsilon}})}\enspace,
\end{align*}
where $H^{-1}(\cdot)$ is the generalized inverse of $H(\cdot)$.
\end{lemma}
Again we conclude that setting $\lambda = H^{-1}(1 - {\bar{\epsilon}})$ in the~\myref{eq:penalized} is equivalent to controlling the average error rate on the level $\bar\epsilon$.

The present framework provides guarantees over the error rate.
These guarantees are weaker theoretically than those of \myref{eq:point_wise_error} and can result in point-wise high error rates.
However, this constraint is easier to enforce and control in practice as the average error rate, unlike its point-wise counterpart, can be measured on a labeled dataset.
We additionally emphasize that the constraint $\Prob(Y \notin \Gamma(\bsX)) \leq \bar{\epsilon} $ actually depends on the whole joint distribution $\Prob$, which stays in contrast with \myref{eq:average_size} formulation.

\myparagraph{Estimation}
To estimate this set-valued classifier, an additional labeled dataset (or subsampling) can be used to fit the threshold.
In particular, denote $(\bsX'_1, Y'_1), \ldots, (\bsX'_{n'}, Y'_{n'})$ this labeled dataset of size $n'$ sampled independently from $\Pr$ and define the following function
\begin{align}
\label{eq:estimationErrorAverageError}
  \forall t \in [0,1], \quad
  \hat{H}(t) = \frac{1}{n'} \sum_{i=1}^{n'} \sum_{\ell=1}^L \ind{ Y'_i = \ell } \; \ind{ \hat{p}_\ell (\bsX'_i) \geq t } .
\end{align}
This function is essentially the empirical counterpart of the function $H$ in Eq.~\eqref{eq:H_func} computed using the labeled dataset and for the estimator $\hat{\bsp}$.
The plug-in set-valued classifier is then defined point-wise by
\begin{align*}
    \forall \bsx \in \X, \quad
    \hat{\Gamma}_{\bar{\epsilon}}(\bsx) = \enscond{\ell \in [L]}{ \hat{p}_\ell(\bsx) \geq \hat{H}^{-1}(1 - {\bar{\epsilon}}) } ,
\end{align*}
where $\hat{H}^{-1}$ is the generalized inverse of $\hat{H}$.

\myparagraph{Bibliographic references}
\citet{Lei2014} studied a similar framework in the context of binary classification with reject option.
This framework was later extended to the multi-class classification framework by \citet{Sadinle2019} which provided the expression of the optimal set-valued classifier.
They also studied convergence properties of the plug-in set-valued classifier under some smoothness assumptions on $\Pr$.
The origin of this framework is rooted in \emph{conformal prediction} literature \citep{Vovk2005} and \citet{Sadinle2019} derives a classification procedure based on the ideas of conformal prediction theory.
\section{Empirical comparison of the frameworks}
\label{sec:numerical}

In this section, we carry out experiments on real-world datasets to analyze and compare the properties of the previously described formulations. We aim at highlighting positive and negatives aspects of the frameworks described in the previous section.




\subsection{Constraint satisfiability in practice}

In this section, we analyze if the constraints defined in each framework are satisfied in practice.
Since $\Prob$ is unknown, for all frameworks, except \myref{eq:top_k} (top-$k$), we cannot guarantee that the constraint will be in a strict accordance with the theoretical requirement. Our conclusions in this section will then be based on empirical evaluation of the different constraints.

We perform our experiments on \texttt{ImageNet} \citep{Russakovsky2015} dataset which contains 1.3 million images from 1,000 classes.
We split this dataset into a training set of 1.2 million images, a validation set of 50,000 images, and a test set of 100,000 images.
We remark that the labels of the official test set are not publicly available and it is thus common practice to use the official validation set as test set for experiments.
This validation set contains 50,000 images of the 1,000 classes.
It is balanced and thus contains 50 images per class.
Recall that the estimation procedures described in the previous section can be performed on top of any off-the-shelf score based classifier. For this reason, we use a pre-trained \texttt{ResNet-152} neural network \citep{He2016} from PyTorch model zoo\footnote{\url{https://pytorch.org/docs/stable/torchvision/models.html}}.
We then use the predictions made on the validation set to carry out the experiments in this section.


\begin{figure}[t]
  \centering
  \begin{subfigure}[b]{.32\textwidth}
    \includegraphics[width=\textwidth]{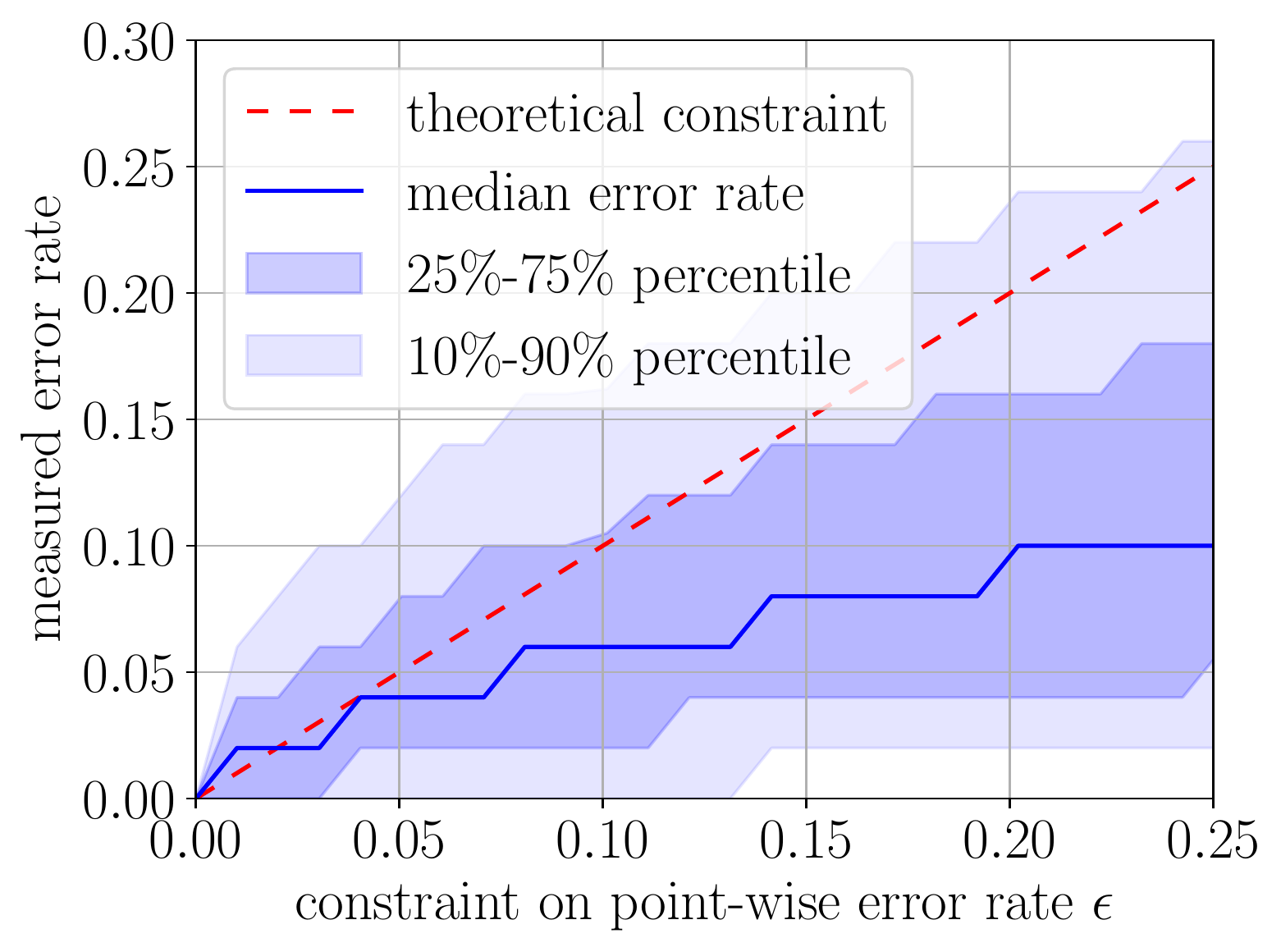}
    \caption{Original probabilities.}
    \label{fig:a_pointwise-error-rate-constraint-estimation}
  \end{subfigure}
  \hfill
  \begin{subfigure}[b]{.32\textwidth}
    \includegraphics[width=\textwidth]{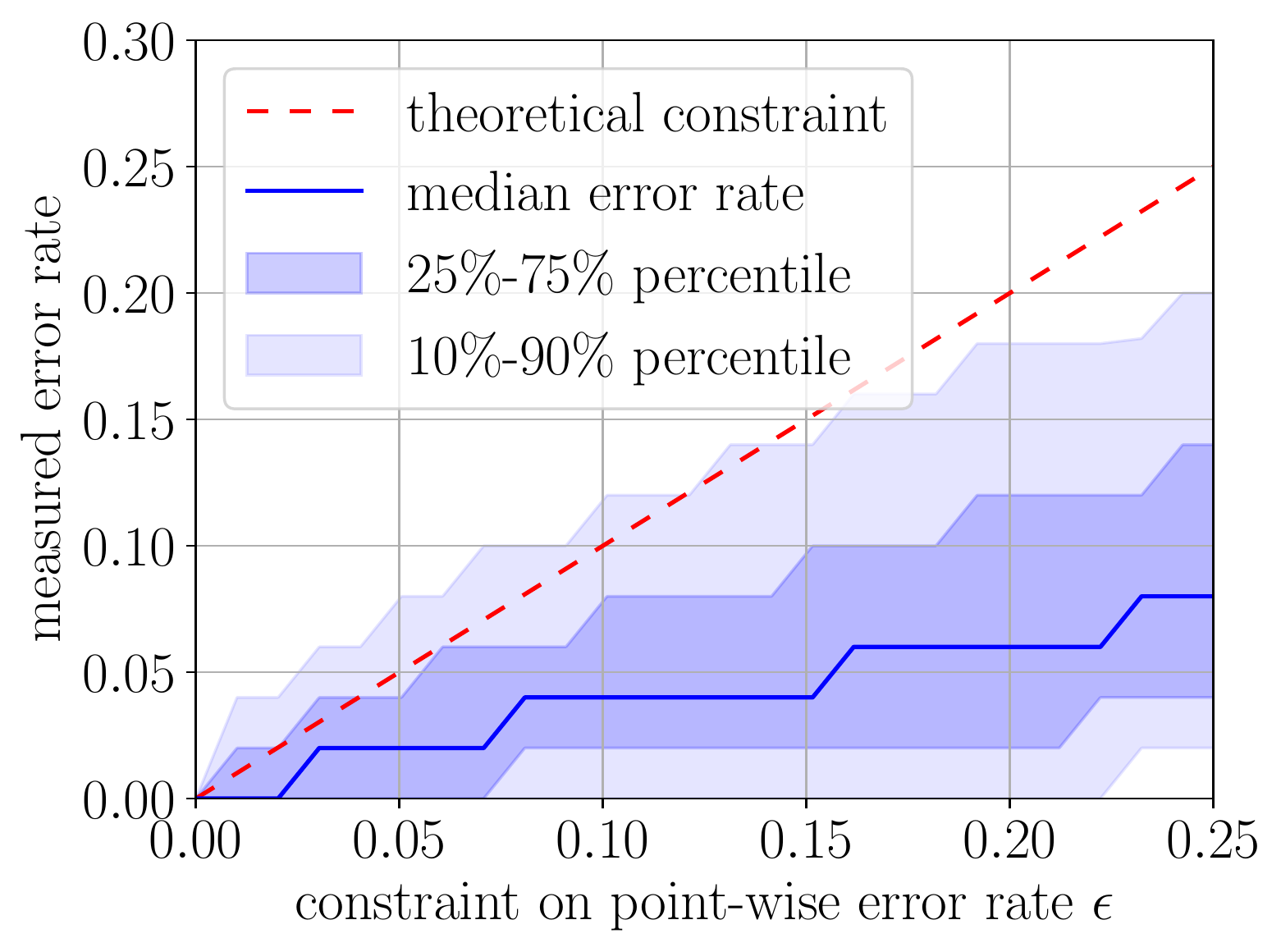}
    \caption{Calibrated probabilities.}
    \label{fig:b_pointwise-error-rate-constraint-estimation}
  \end{subfigure}
  \hfill
  \begin{subfigure}[b]{.32\textwidth}
    \includegraphics[width=\textwidth]{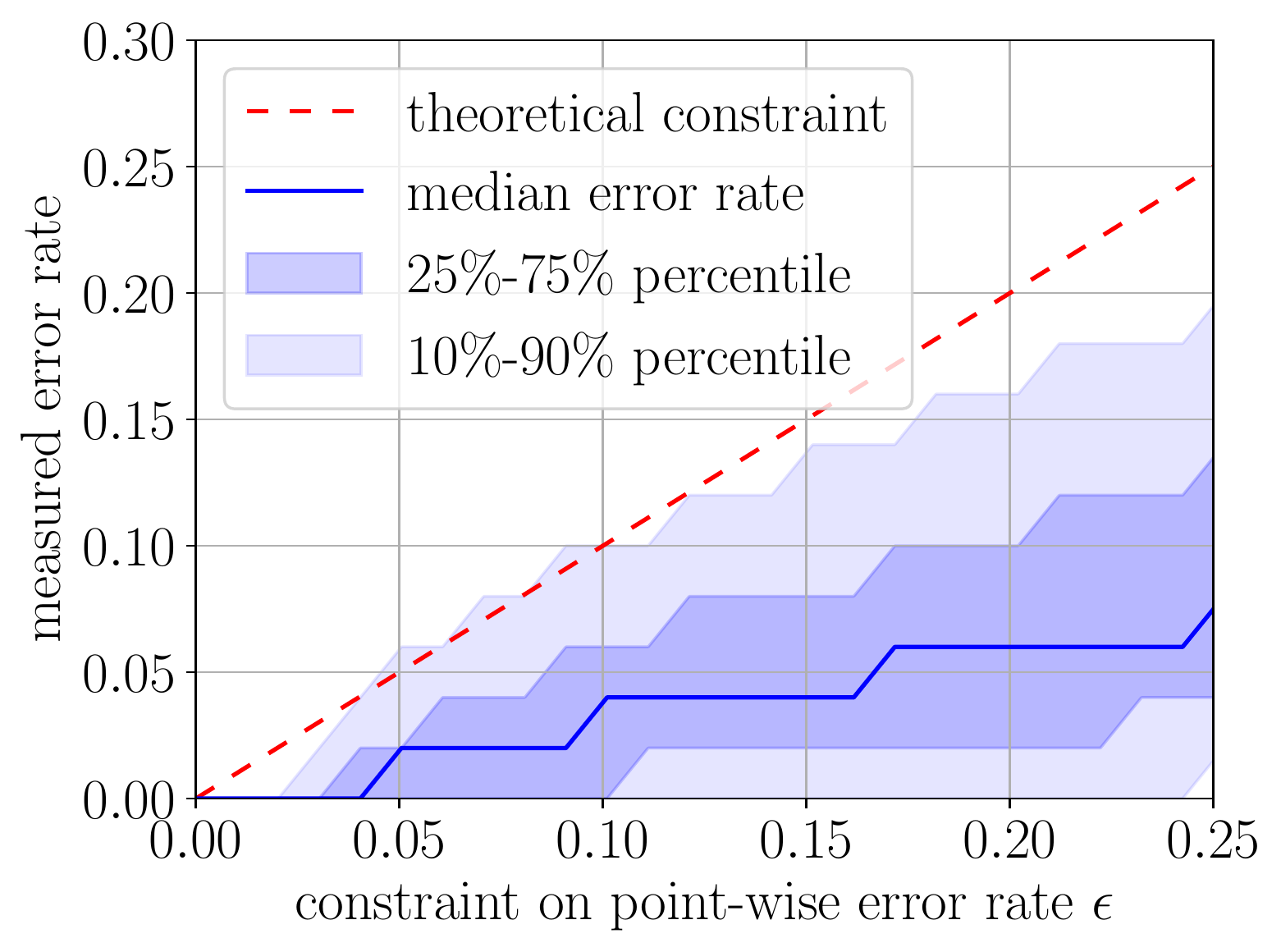}
    \caption{Prob. calib. + correction}
    \label{fig:c_pointwise-error-rate-constraint-estimation}
  \end{subfigure}
\caption{
  Plots assessing if the point-wise error rate constraint is violated on \dataset{ImageNet}.
  The distributions of the class error rates are computed for different values of point-wise error rate constraints.
  The median and the 10\%, 25\%, 75\%, and 90\% percentiles of the distribution of the class error rates are shown.
  If the constraint would have been satisfied, these curves would be below the red curve.
  These plots show that probability calibration is important for the \myref{eq:point_wise_error} formulation, however, it does not guarantee that the constraint will be satisfied.
}
\label{fig:pointwise-error-rate-constraint-estimation}
\end{figure}

\subsubsection{Satisfiability of the point-wise error rate constraint}

We start with the \myref{eq:point_wise_error} framework.
In this case, it is difficult to provide a good evaluation of the point-wise error $\Prob(Y \notin \Gamma(\bsX) \mid \bsX = \bsx)$ involved in the constraint. Indeed, for a specific $\bsx \in  \class{X}$ we have only access to a binary variable $\ind{Y \notin \Gamma(\bsx)}$. Hence, unless we have additional assumptions on the distribution of $(\ind{Y \notin \Gamma(\bsX)}, \bsX)$ it is difficult to estimate the conditional expectation $\Exp[ \ind{Y \notin \Gamma(\bsX)}\mid \bsX] = \Prob(Y \notin \Gamma(\bsX) \mid \bsX )$.
Instead, we report a weaker quantity -- the class conditional error rate $\Prob(Y \notin \Gamma(\bsX) \mid Y = y)$ for all $y\in [L]$.

As shown in \Cref{fig:pointwise-error-rate-constraint-estimation}-(a), the constraint is violated when we consider a neural network based estimator of the probabilities. It is however possible to enforce the constraint with additional steps:
\emph{i) calibration.} We apply a probability calibration strategy that requires an additional dataset, and which is well suited for neural networks \citep{Guo2017}. This step consists in learning a temperature parameter $T$ which is used to scale the predictions in the logit space before applying the softmax. As displayed in \Cref{fig:pointwise-error-rate-constraint-estimation} (mid), this approach provides a good correction in terms of constraint violation. However, even with this calibration step, the constraint is still violated; \emph{ii) correction.} We consider the correction step described given by \Cref{eq:correction_pointwise_error}. The effect of the calibration and the correction is displayed in \Cref{fig:pointwise-error-rate-constraint-estimation} (right) where we observe that the constraint in this case is satisfied up to $90\%$ percentile.

Apart from the recalibration, we conclude that for the particular \myref{eq:point_wise_error} formulation, having properly calibrated probabilities is important.


This experiment shows that point-wise error constraint is neither easy to satisfy nor easy to measure in practice.
Moreover, although theoretically, we do not need an extra dataset to learn this constraint, in practice using this additional data to calibrate the probabilities is very helpful and can improve the performance.
In the rest of the experiments, we will thus use temperature scaling to calibrate probabilities when studying this formulation.

\subsubsection{Satisfiability of average constraints}

\begin{figure}[t]
  \centering
  \begin{subfigure}[b]{.48\textwidth}
    \includegraphics[width=\textwidth]{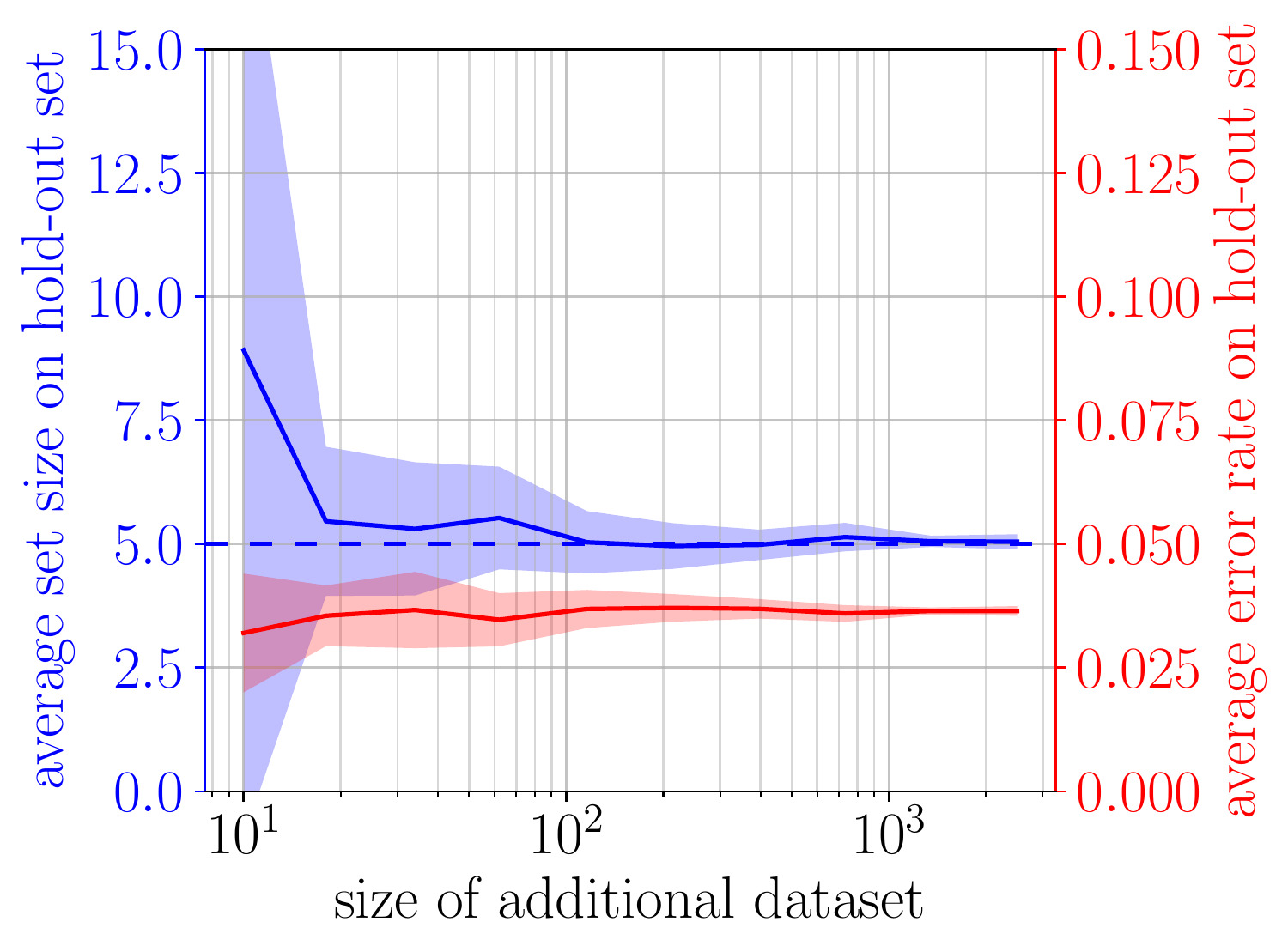}
    \caption{Average size control with $\bar{k}=5$.}
  \end{subfigure}
  \hfill
  \begin{subfigure}[b]{.48\textwidth}
    \includegraphics[width=\textwidth]{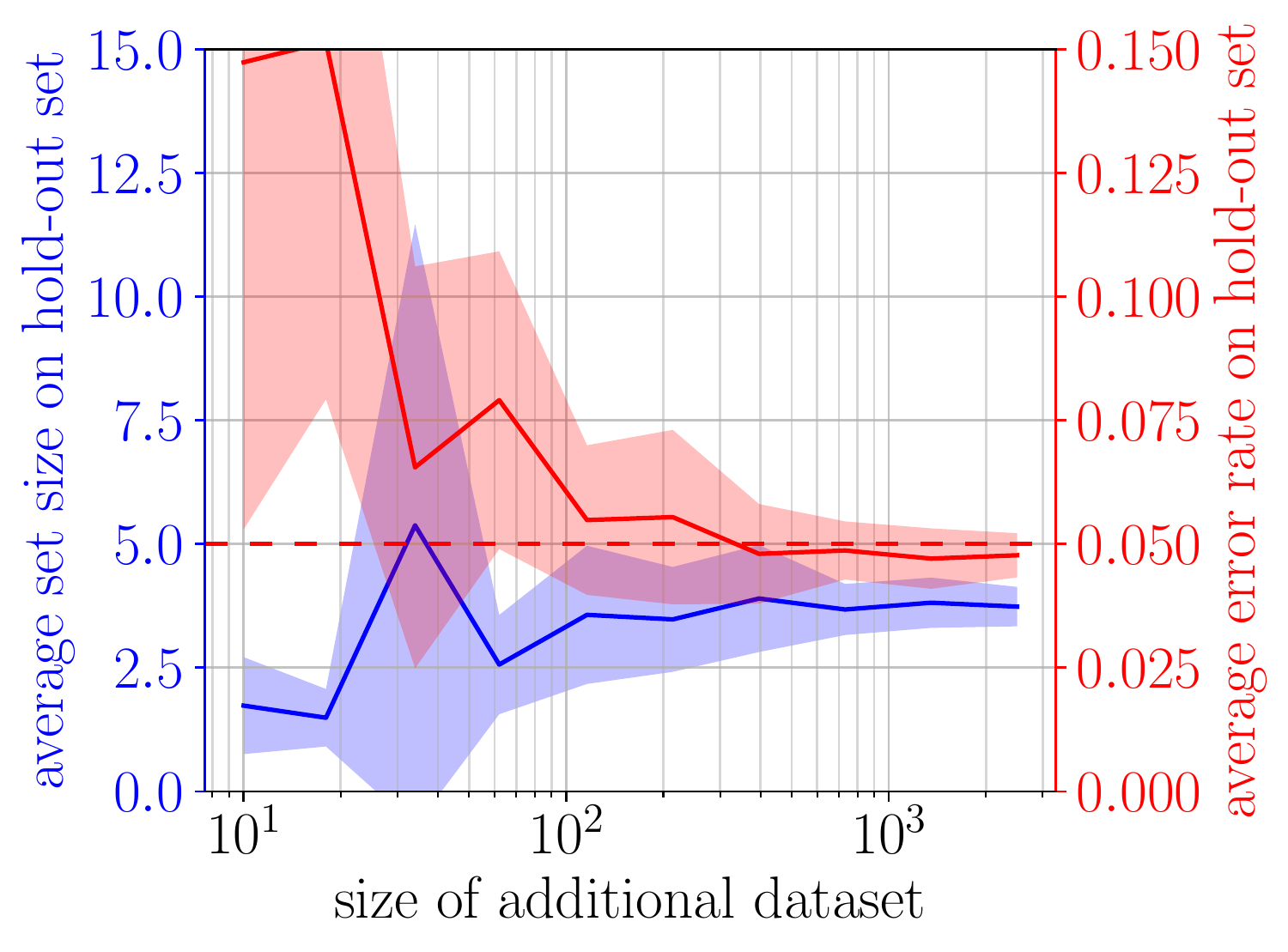}
    \caption{Average error control with $\bar{\epsilon}=0.05$.}
  \end{subfigure}
  \caption{
    Estimation plots depending on the size of the additional dataset of the thresholds of the average size control and average error control formulations on \dataset{ImageNet}.
    The mean values are shown in full line while the shaded areas correspond to the standard deviation.
    The dashed line represents the constraint to estimate.
    Average size control requires less additional data than average error control to estimate a good threshold.
  }
  \label{fig:average-constraints-estimation}
\end{figure}

We now focus on constraint violation for \myref{eq:average_error} and \myref{eq:average_size} frameworks.
In both cases, we first need to compute a threshold that we estimate thanks to an additional dataset\footnote{One can use the same dataset to estimate the threshold and to evaluate the constraint, however, such an approach would lead to dependency issues.}. The size of this sample is denoted by $n'$ and $N$ for the \myref{eq:average_error} and \myref{eq:average_size} frameworks respectively, see \Cref{eq:estimationCDFAverageSize} and \Cref{eq:estimationErrorAverageError} respectively.
In what follows, we illustrate the effect of the sample sizes $n'$ and $N$ on the corresponding constraints. Secondly, we need another dataset to measure the average error rate and the average set size according to the considered framework. Therefore, we split the validation set into two sets of equal sizes. One part is then used to sample the additional dataset while the second one serves as a hold-out set to measure the average set size and average error rate.
The sampling of the additional dataset is performed $10$ times to measure standard deviations.
The results are shown in \Cref{fig:average-constraints-estimation}.

For \myref{eq:average_size} framework, the average size of the predicted sets lies within an interval around the fixed constraint and the convergence is fast \wrt $N$. Interestingly, even the average error rate of those sets is very stable and converges fast as well.
In contrast, for \myref{eq:average_error} framework, both of the average error rate and the set size are highly oscillating and converge significantly more slowly. It seems in addition that the error rate, which is the parameter of interest here, is biased towards higher error rates for small values of $n'$.
The main conclusion drawn from these plots is thus that the calibration of threshold requires fewer samples for \myref{eq:average_size} than for \myref{eq:average_error}. This is in accordance with the general intuition since the error rate depends on the distribution of $(\bsX,Y)$ while the set size relies only on the distribution of $\bsX$. As a consequence more data are needed to ensure good estimation of all the functions $ t \mapsto  \Prob (p_\ell (\bsX) > t)  $ for all $\ell \in[L]$ in the same time.

\subsection{Comparison on real-world datasets}

We now study the different formulations on different real-world datasets and analyze their properties.
We focus on three image classification datasets: \texttt{MNIST} \citep{LeCun1998}, \texttt{ImageNet} \citep{Russakovsky2015} and \texttt{PlantCLEF2015} \citep{Goeau2015}.

\begin{figure}[t]
  \foreach \dataset / \name in {mnist/\texttt{MNIST},imagenet/\texttt{ImageNet},plantclef2015/\texttt{PlantCLEF2015}} {
    \begin{subfigure}{.32\textwidth}
          \includegraphics[width=\textwidth]{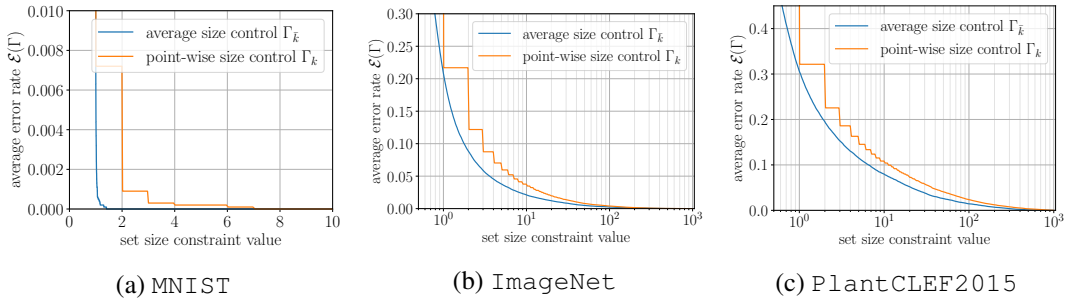}
      \caption{\name}
    \end{subfigure}
    \hfill
  }
  \caption{
    Comparison of \myref{eq:average_size} and \myref{eq:top_k} formulations, \ie average error rate minimization under set size constraint formulations, on different datasets.
    The average error rate is plotted against the value of the constraint for both formulations, \ie $\bar{k}$ for the average control and $k$ for the point-wise control.
  }
  \label{fig:average-pointwise-size-control-comparison}
\end{figure}

We first compare the formulations sharing the same objective but with a different constraint.
In \Cref{fig:average-pointwise-size-control-comparison}, we compare \myref{eq:average_size} and \myref{eq:top_k} formulations by plotting the achieved average error rate against the value of the constraint on the set size.
Note that the curve of \myref{eq:top_k} is a piece-wise constant function as the point-wise set size can only take discrete values.
The average set size does not have this constraint and thus the curve of \myref{eq:average_size} is continuous.
As can be seen from the figure, the curve of \myref{eq:average_size} is always below the one of \myref{eq:top_k}.
From left to right, the two curves are less and less steep.
This suggests that those datasets have a different amount of ambiguity: \texttt{MNIST} is far less ambiguous than \texttt{PlantCLEF2015}.
Moreover, the wideness of the gap between the two curves depends on the dataset.

\begin{figure}[t!]
  \begin{subfigure}{\textwidth}
    \foreach \xmax in {1.0,0.005} {
      \begin{subfigure}[b]{.45\textwidth}
        \includegraphics[width=\textwidth]{imgs/experiments/mnist_min_avg_set_size_const_error_rate_xmax_\xmax}
      \end{subfigure}
      \hfill
    }
    \caption{\dataset{MNIST}}
  \end{subfigure}
  \vskip
  \baselineskip
  \begin{subfigure}{\textwidth}
    \foreach \xmax in {1,0.2} {
      \begin{subfigure}[b]{.45\textwidth}
        \includegraphics[width=\textwidth]{imgs/experiments/imagenet_min_avg_set_size_const_error_rate_xmax_\xmax}
      \end{subfigure}
      \hfill
    }
    \caption{\dataset{ImageNet}}
  \end{subfigure}
  \vskip
  \baselineskip
  \begin{subfigure}{\textwidth}
    \foreach \xmax in {1,0.2} {
      \begin{subfigure}[b]{.45\textwidth}
        \includegraphics[width=\textwidth]{imgs/experiments/plantclef2015_min_avg_set_size_const_error_rate_xmax_\xmax}
      \end{subfigure}
      \hfill
    }
    \caption{\dataset{PlantCLEF2015}}
  \end{subfigure}
  \caption{
    Comparison of \myref{eq:average_error} and \myref{eq:point_wise_error} formulations, \ie average set size minimization under error rate constraint formulations, on different datasets.
    The average set size is plotted against the value of the constraint for both formulations, \ie $\bar{\epsilon}$ for the average control and $\epsilon$ for the point-wise control.
    The complete plots are displayed on the left column while the right one shows a zoom on a subpart of the plot.
  }
  \label{fig:average-pointwise-error-control-comparison}
\end{figure}

Similarly, in \Cref{fig:average-pointwise-error-control-comparison}
, we compare \myref{eq:average_error} and \myref{eq:point_wise_error} formulations by plotting the achieved average set size against the value of the constraint on the error rate.
As can be noticed, \myref{eq:average_error} allows to reach average set sizes below one whereas \myref{eq:point_wise_error} necessarily predicts sets of size larger or equal to one.
Theoretically, if the constraints were properly enforced, the curve of \myref{eq:point_wise_error} should always be above the curve of \myref{eq:average_error}.
However, as we discussed in the previous section, the point-wise constraint is hard to guarantee in practice and thus this property can be violated as it is the case on \texttt{PlantCLEF2015}.
This may suggest that the models on this dataset are unable to accurately estimate $\bsp$.
As with the previous figure, these plots also give us insight into the amount of ambiguity in the datasets.
In particular, we can see that it is easy to achieve a very low point-wise error rate on \texttt{MNIST} while keeping an average set size below two.
Therefore \myref{eq:point_wise_error} formulation is useful for \texttt{MNIST} if strong guarantee on the error rate is desired.
On the other hand, for \texttt{ImageNet}, enforcing a low point-wise error rate implies predicting sets of size far larger than $10$ on average which becomes not informative in most cases.
Constraining the average error rate, however, allows predicting sets far smaller.
Thus, for \texttt{ImageNet}, the \myref{eq:average_error} formulation might be preferred.

\begin{figure}[t!]
  \foreach \dataset / \name in {mnist/MNIST,imagenet/ImageNet,plantclef2015/PlantCLEF2015} {
    \begin{subfigure}{.32\textwidth}
      \includegraphics[width=\textwidth]{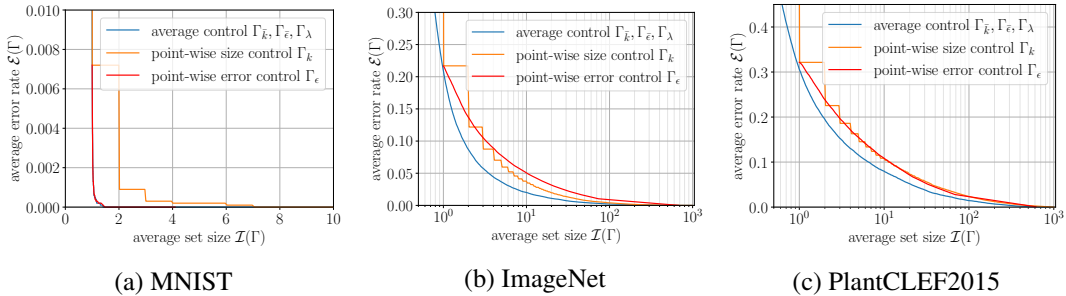}
      \caption{\name}
    \end{subfigure}
    \hfill
  }
  \caption{
    Comparison of the different frameworks in terms of their average set size / average error rate curves.
    In this case, all average control frameworks share the same optimal curve.
    The curves of the other frameworks are necessarily above this curve.
  }
  \label{fig:average-size-error-comparison}
\end{figure}

In \Cref{fig:average-size-error-comparison}, we compare all the frameworks to each other by plotting their achieved average error rate against their measured average set size when varying their constraints.
On this plot, the average formulations share the same curve which is also the optimal curve (since the constraint for one formulation is the objective of the other).
Interestingly, for the different datasets, the two point-wise control formulations behave differently.
On \texttt{MNIST}, the \myref{eq:point_wise_error} is very close to the optimal curve, while on \texttt{ImageNet}, it is above the \myref{eq:top_k} curve, and on \texttt{PlantCLEF2015}, the curves of the point-wise control formulation coincide.
This is in line with our previous comment on the usefulness of \myref{eq:point_wise_error} formulation for \texttt{MNIST} and its limitations for \texttt{ImageNet}.

\begin{figure}[t!]
  \centering
  \begin{subfigure}[b]{\textwidth}
    \includegraphics[width=\textwidth]{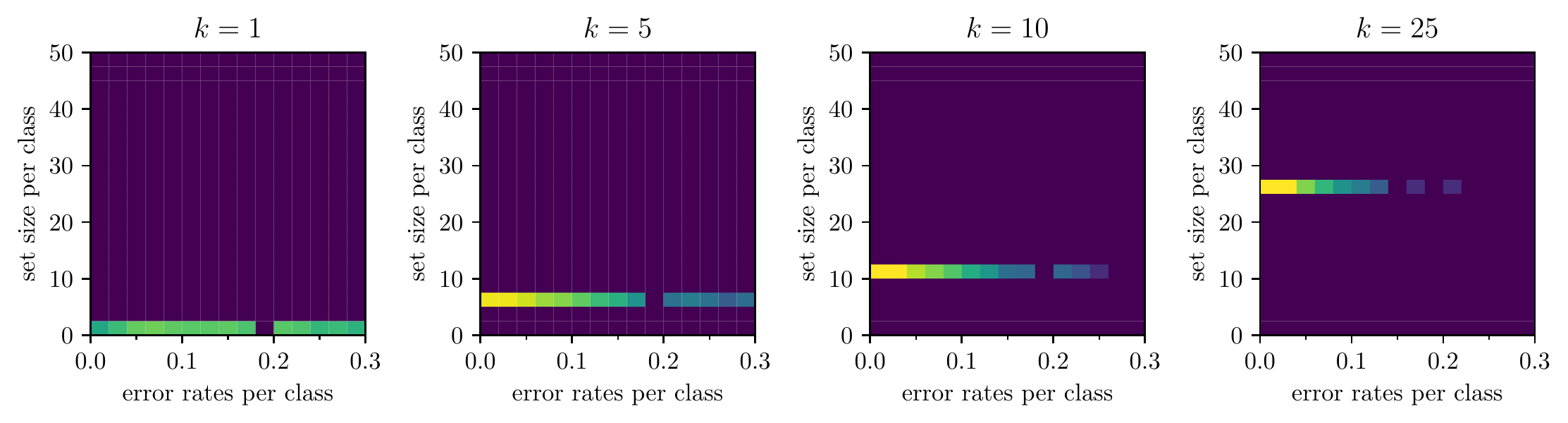}
    \caption{
      Point-wise set size control.
    }
  \end{subfigure}
  \\
  \begin{subfigure}[b]{\textwidth}
    \includegraphics[width=\textwidth]{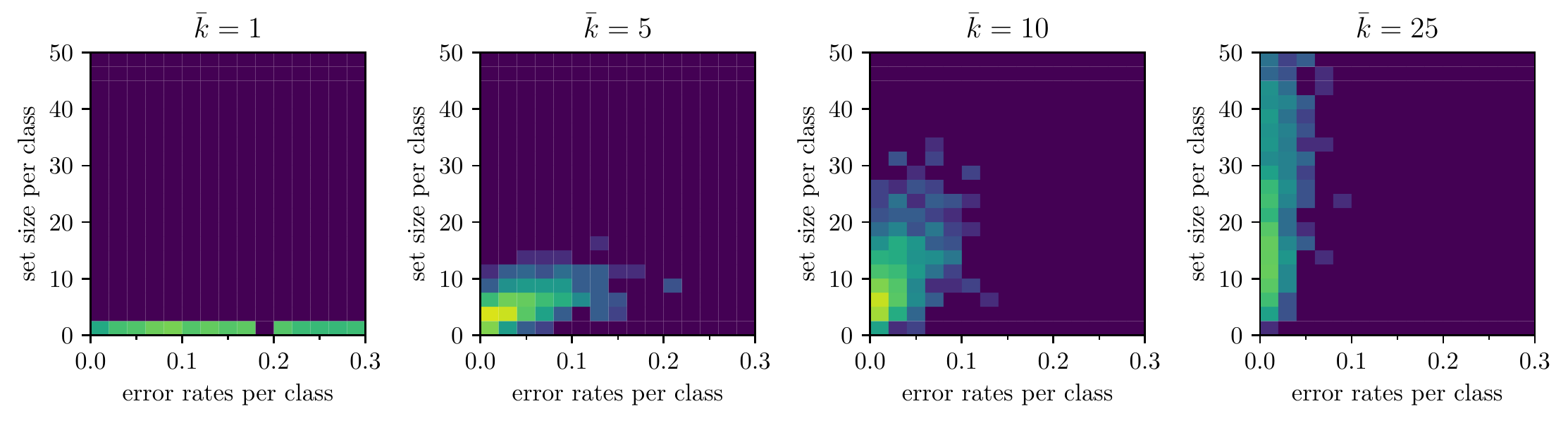}
    \caption{
      Average set size control.
    }
  \end{subfigure}
  \\
  \begin{subfigure}[b]{\textwidth}
    \includegraphics[width=\textwidth]{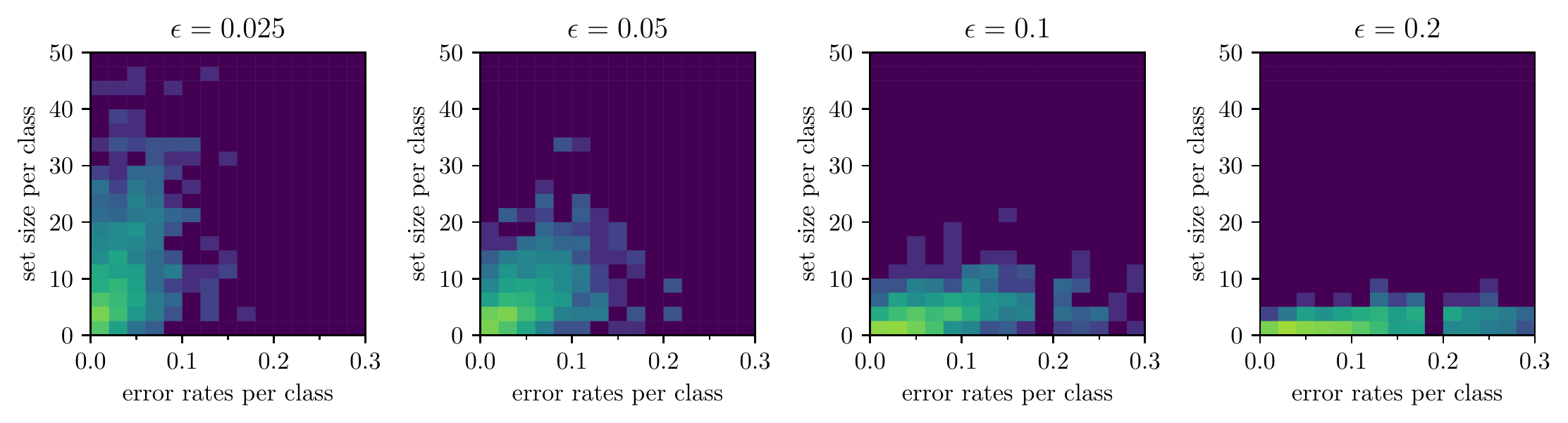}
    \caption{
      Point-wise error rate control.
    }
  \end{subfigure}
  \\
  \begin{subfigure}[b]{\textwidth}
    \includegraphics[width=\textwidth]{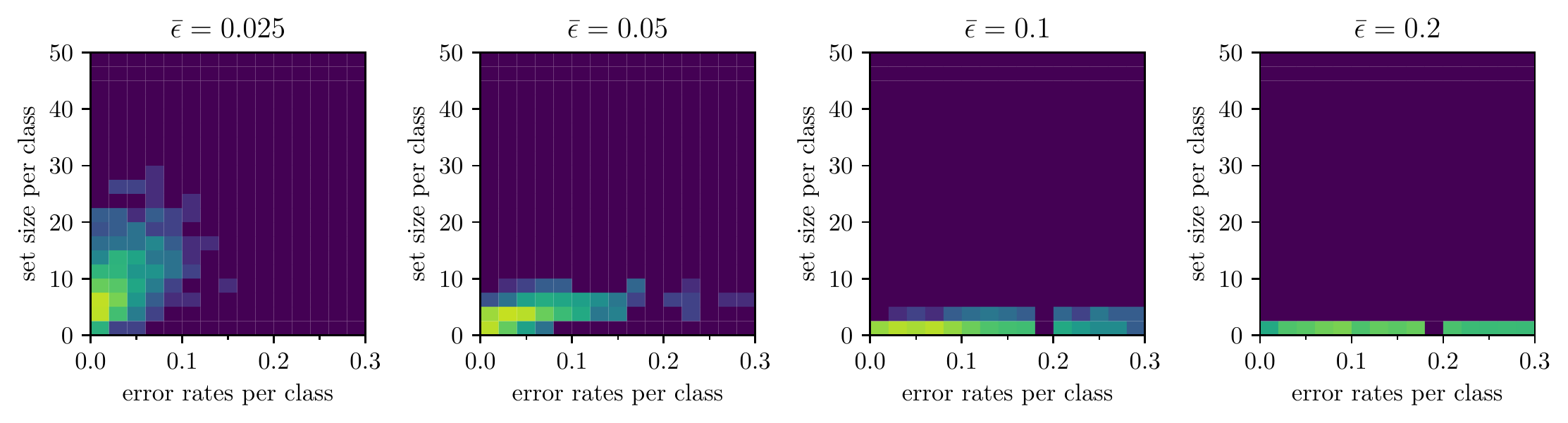}
    \caption{
      Average error rate control.
    }
  \end{subfigure}
  \caption{
    Comparison of the different formulations by looking at the distribution of the error rates and set sizes per class (on \dataset{ImageNet}).
    The plots share the same color scale.
  }
  \label{fig:pointwise-size-error-comparison}
\end{figure}

Finally, in \Cref{fig:pointwise-size-error-comparison}, we compare the distributions of the point-wise set sizes and point-wise error rates for the different formulations.
These distributions are measured on \texttt{ImageNet} in a similar fashion than in the previous subsection: the error rates and set sizes are computed for each class and serve as proxies for the point-wise error rate and set size.
These distributions are plotted using 2D histograms for various values of the constraints.
To compare them, all the plots share the same color scale.
Compared to the other frameworks, \myref{eq:top_k} have a particular behavior: it allows a strict control on the point-wise set size with all the samples having exactly the same set size.
The class error rates are thus always distributed along a line for this framework.
Comparing with \myref{eq:average_size}, we can see how the relaxation of the point-wise constraint allows reducing the error rate by predicting bigger sets for samples of higher error rate and smaller ones for samples with lower error rate.
This remark is true for all formulations (except \myref{eq:top_k}).
Indeed, when transitioning from a weak to a strong constraint, we can see that the point-wise error rates and the point-wise set size are correlated: the higher the error rate, the more labels are predicted, and vice-versa.
To conclude, by focusing on the error control frameworks, we can notice that the set sizes are more dispersed for \myref{eq:point_wise_error} than for \myref{eq:average_error} at equal constraint values.

These different experiments highlight the similarities and differences between the described set-valued classification frameworks.

\section{Other set-valued frameworks}
\label{sec:discussion_other_framework}
Previous sections were dealing with constrained formulations with the constrains of similar nature, that is, the constraints were either point-wise or in expectation.
In principle, we can combine both point-wise and expected constraints in one to retain positive aspects of both as well as modify the objective function which trades-off the two adaptively.
In this section we present some natural strategies and discuss their features.
\subsection{Hybrid frameworks}
\label{sec:hybrids}

In what follows, we describe possible ways to combine constraints of different types highlighting the issues that can arise from not a careful choice of the tuning parameters.

\subsubsection{Hybrid size control}
Minimizing the error given constraint on the expected set size and upper bound on point-wise: given $k \in \mathbb{N}, \bar{k} \in \mathbb{R}$ such that $0 < \bar{k} < k \leq K$,
\usetagform{bold}
\begin{equation*}
    \tag{hybrid size control}
        \begin{aligned}
            \Gamma^*_{\bar{k},k} \in \argmin_\Gamma \; & \Prob(Y \notin \Gamma(\bsX) ) \\
            \text{s.t. } & \Exp_{\bsX} |\Gamma(\bsX)| \leq \bar{k} \\
            & \phantom{\Exp_{\bsX}}|\Gamma(\bsX)| \leq k, \text{ a.s. } \Prob_{\bsX}
        \end{aligned}
\end{equation*}
\usetagform{default}

\begin{lemma}
    \label{lem:hybrid_size_control_oracle}
    Let Assumption~\ref{ass:continuity_global} be satisfied. Fix $\bar{k} < k$ and define
    \begin{align*}
        G_k(t) = \sum_{\ell = 1}^k \Prob \left( p_{(\ell)}(\bsX) \geq t \right)\enspace.
    \end{align*}
    Then, an optimal set-valued classifier $\Gamma^*_{\bar{k}, k}$ can be obtained, for all $\bsx \in \mathcal{X}$ as
    \begin{align*}
        \Gamma^*_{\bar{k}, k}(\bsx) = \enscond{\ell \in [L]}{p_{\ell}(\bsx) \geq G_k^{-1}(\bar{k})} \bigcap \Top_{\bsp}(\bsx, k)\enspace,
    \end{align*}
    where $G_k^{-1}$ is the generalized inverse of $G_k$.
\end{lemma}

\myparagraph{Pros and Cons}
Stronger set size control than for other formulations. Higher error rate than for soft set size control.

\subsubsection{Hybrid error control}
Minimizing the expected set size given constraint on the expected error rate and upper bound on point-wise: given $0 \leq \bar{\epsilon} < \epsilon \leq 1$,
\usetagform{bold}
\begin{equation*}
        \tag{hybrid error control}
        \begin{aligned}
            \Gamma^*_{\bar{\epsilon}, \epsilon} \in \argmin_\Gamma \; &  \Exp_{\bsX} |\Gamma(\bsX)| \\
            \text{s.t. } & \Prob(Y \notin \Gamma(\bsX)) \leq \bar{\epsilon} \\
            & \Prob(Y \notin \Gamma(\bsX) | \bsX) \leq \epsilon,  \text{ a.s. } \Prob_{\bsX}
        \end{aligned}
\end{equation*}
\usetagform{default}

\begin{lemma}
    \label{lem:size_size_hybrid}
    Let Assumption~\ref{ass:continuity_global} be satisfied. For $\bar{\epsilon} < \epsilon$ define
    \begin{align*}
        H_\epsilon(\cdot) = \Exp_{\bsX}\left[\sum_{\ell = 1}^{k_\epsilon(\bsX)}p_{(\ell)}(\bsX)\ind{ p_{(\ell)}(\bsX) \geq \cdot}\right]\enspace,
    \end{align*}
    where $k_\epsilon(\bsX)$ is defined in Lemma~\ref{lem:hard_coverage_control_oracle}.
    Then if $H_\epsilon(1 - \bar{\epsilon}) \neq 0$, then
    \begin{align*}
        \Gamma^*_{\bar{\epsilon}, \epsilon}(\bsx) = \enscond{\ell \in [L]}{p_{\ell}(\bsx) \geq H^{-1}_\epsilon(1 - \bar{\epsilon})}\enspace,
    \end{align*}
    where $H^{-1}_\epsilon$ is the generalized inverse of $ H_\epsilon$.
\end{lemma}
\myparagraph{Pros and Cons} Stronger set coverage control than for other formulations. Higher average set size than for soft set coverage control.

\subsubsection{Controlling error and size}

\usetagform{bold}
\begin{equation}
    \tag{average error + point-wise size}
    \begin{aligned}
        \Gamma^*_{\bar{\epsilon}, k} \in \argmin_\Gamma \; &  \Exp_{\bsX} |\Gamma(\bsX)| \\
        \text{s.t. }
            & \Prob(Y \notin \Gamma(\bsX)) \leq \bar{\epsilon} \\
            & |\Gamma(\bsX)| \leq k,  \text{ a.s. } \Prob_{\bsX}
    \end{aligned}
\end{equation}
\usetagform{default}

\begin{lemma}
    Define $\epsilon_k$ as
    \begin{align*}
        \epsilon_k = \Prob(Y \notin \Top_{\bsp}(\bsX, k))\enspace.
    \end{align*}
    Fix $k \in [L]$, then if $\bar{\epsilon} < \epsilon_k$ the feasible set of the above problem is empty.
\end{lemma}
\begin{proof}
    Assume the opposite, that is, there exists $\Gamma$ such that
    \begin{align*}
        \Prob(Y \notin \Gamma(\bsX)) \leq \bar{\epsilon},\quad |\Gamma(\bsX)| \leq k,  \text{ a.s. } \Prob_{\bsX}\enspace.
    \end{align*}
    Recall that $\Top(\bsX, k)$ is an optimal set-valued classifier with hard size constraints, but
    \begin{align*}
        \Prob(Y \notin \Gamma(\bsX)) \leq \bar{\epsilon} < \epsilon_k = \Prob(Y \notin \Top(\bsX, k))\enspace,
    \end{align*}
    which contradicts the optimality of $\Top(\bsX, k)$.
\end{proof}
\myparagraph{Pros and Cons} Requires careful parameter tuning.

\usetagform{bold}
\begin{equation*}
    \tag{average size + point-wise error}
    \begin{aligned}
        \Gamma^*_{\bar{k}, \epsilon} \in \argmin_\Gamma \; &  \Prob(Y \notin \Gamma(\bsX)) \\
        \text{s.t. }
            & \Exp_{\bsX} |\Gamma(\bsX)| \leq \bar{k} \\
            & \Prob(Y \notin \Gamma(\bsX) \mid \bsX) \leq \epsilon,  \text{ a.s. } \Prob_{\bsX}
    \end{aligned}
\end{equation*}
\usetagform{default}

\myparagraph{Pros and Cons} Requires careful parameter tuning.

\subsection{F-score}
\label{subsec:f-score}
While constraint formulations are attractive when an explicit control on the size or coverage is required, some problems might require to optimize a certain trade-off of the two.
Given a set-valued classifier $\Gamma$ we define its precision and recall as
\begin{align*}
    \precision(\Gamma) = \frac{\Prob(Y \in \Gamma(\bsX))}{\Exp|\Gamma(\bsX)|},\quad \recall(\Gamma) = \Prob(Y \in \Gamma(\bsX))\enspace,
\end{align*}
respectively.
Relying on this notions of precision and recall we can define F$_{\beta}$ score as a weighted harmonic average of the latter
\begin{align*}
    F_\beta(\Gamma) = \left(\frac{1}{1 + \beta^2}\precision^{-1}(\Gamma) + \frac{\beta^2}{1 + \beta^2}\recall^{-1}(\Gamma)\right)^{-1} = \frac{(1 + \beta)\Prob(Y \in \Gamma(\bsX))}{\beta^2+\Exp |\Gamma(\bsX)|}\enspace.
\end{align*}
Consequently, an optimal set-valued classifier in terms of the F$_\beta$ score is defined as
\usetagform{bold}
\begin{align*}
    \tag{F-score}
    \Gamma^*_{\text{F}} \in \argmax_\Gamma \frac{(1 + \beta^2)\Prob(Y \in \Gamma(\bsX) )}{\beta^2+\Exp|\Gamma(\bsX)|}\enspace.
\end{align*}
\usetagform{default}

\begin{lemma}
    Fix $\beta > 0$, then an optimal set-valued classifier $\Gamma^*_{\text{F}}$ can be obtained for all $\bsx \in \class{X}$ as
    \begin{align*}
        \Gamma^*_{\text{F}}(\bsx) = \enscond{\ell \in [L]}{p_{\ell}(\bsx) \geq \theta^*}\enspace,
    \end{align*}
    where $\theta^*$ is a unique root of
    \begin{align*}
        \theta \mapsto \beta^2 \theta - \sum_{\ell = 1}^L\Exp(p_{\ell}(\bsX) - \theta)_+\enspace,
    \end{align*}
    where $(a)_+$ denotes the positive part of $a$, \ie $(a)_+ = \max(a,0)$.
\end{lemma}

\myparagraph{Pros and Cons} Link with information retrieval metrics. The quality measured by one score functions. Less interpretable than constrained versions.

\section{Conclusion}
In this manuscript we have provided a review of set-valued classification literature using a unified framework of minimization under possibly distribution dependent constraints. We highlighted two main trade-offs that should be considered -- the size and the error. Both quantities can be measured point-wise or in average leading to various natural formulations of optimal set-valued classifiers.
We emphasized pros and cons of the described frameworks and provided empirical studies supporting our conclusions numerically.

\newpage
\appendix

\section{Omitted proofs}
\label{sec:ProofOptimalPopular}

In this appendix, we gather the proofs omitted in Section~\ref{sec:PopularSvFramework}.

\begin{proof}[Proof of Lemma~\ref{lem:hard_coverage_control_oracle}]
    First of all, notice that for all $\Gamma$ it holds that
    \begin{align*}
        \Prob(Y \notin \Gamma(X) | X)
        &= 1 - \Prob(Y \in \Gamma(X) | X=x)\\
        &= 1 - \sum_{k = 1}^K \Exp\left[\ind{k \in \Gamma(X)} \ind{Y = k}| X\right]\\
        &=
        1 - \sum_{k = 1}^K p_k(X)\ind{k \in \Gamma(X)}\enspace.
    \end{align*}
    Thus, the constraint
    \begin{align*}
        \Prob(Y \notin \Gamma(X) | X) \leq \epsilon,\quad\text{a.s.}\enspace,
    \end{align*}
    is equivalent to
    \begin{align*}
        \sum_{k = 1}^K p_k(X)\ind{k \in \Gamma(X)} \geq 1 - \epsilon\enspace.
    \end{align*}
    The above implies that $\Gamma^*_{\epsilon}$ can be obtained point-wise for all $x \in \mathcal{X}$ as
    \begin{align*}
        \Gamma^*_{\epsilon}(x) \in \argmin_{\gamma \subset [K]} & \sum_{k = 1}^K \ind{k \in \gamma} \\
        \text{s.t. } & \sum_{k \in \gamma} p_k(x) \geq 1 - \epsilon\enspace.
    \end{align*}
    The proof is concluded by solving the above problem.
\end{proof}

\begin{proof}[Proof of Lemma~\ref{lem:penalized_oracle}]
Fix $\lambda > 0$, the objective function $U(\Gamma) = \Prob(Y \notin \Gamma(\bsX) ) +  \lambda \  {\Exp_{\bsX}}|\Gamma(\bsX)|$ to be minimized can be written as
    \begin{align*}
        U(\Gamma)
        &=
        1 + \Exp_{\bsX}\sum_{\ell = 1}^L \left[(\lambda - p_{\ell}(\bsX))\ind{\ell \in \Gamma(\bsX)}\right]\enspace.
    \end{align*}
    Clearly, minimizing the above objective can be performed point-wise, that is, for all $\bsx \in \class{X}$ it holds that
    \begin{align*}
        \Gamma^*_{\lambda}(\bsx) = \enscond{\ell \in [L]}{p_{\ell}(\bsx) \geq \lambda}\enspace.
    \end{align*}
\end{proof}

\begin{proof}[Proof of Lemma~\ref{lem:hybrid_size_control_oracle}]
    Fix some $\bar{k} < k$.
    We are interested in a minimizer of the following problem
    \begin{align*}
        (*) = \min\enscond{\Prob(Y \notin \Gamma(\bsX) )}{\Exp_{\bsX} |\Gamma(\bsX)| \leq \bar{k},\,\,|\Gamma(\bsX)| \leq k, \text{ a.s. } \Prob_{\bsX}}\enspace.
    \end{align*}
    First of all notice that thanks to the weak duality we have
    \begin{align*}
        (*)
        &=
        \min_{\Gamma}\max_{\lambda \geq 0}\enscond{\Prob(Y \notin \Gamma(\bsX) ) + \lambda(\Exp_{\bsX} |\Gamma(\bsX)| - \bar{k})}{|\Gamma(\bsX)| \leq k, \text{ a.s. } \Prob_{\bsX}}\\
        &\geq
        \max_{\lambda \geq 0}\min_{\Gamma}\enscond{\Prob(Y \notin \Gamma(\bsX) ) + \lambda(\Exp_{\bsX} |\Gamma(\bsX)| - \bar{k})}{|\Gamma(\bsX)| \leq k, \text{ a.s. } \Prob_{\bsX}}\enspace.
    \end{align*}
    The objective function of the maxmin problem can be expressed as
    \begin{align*}
        \Prob(Y \notin \Gamma(\bsX) ) + \lambda(\Exp_{\bsX} |\Gamma(\bsX)| - \bar{k}) = 1 - \lambda \bar{k} + \Exp_{\bsX}\left[\sum_{\ell = 1}^L(\lambda - p_{\ell}(\bsX))\ind{\ell \in \Gamma(\bsX)}\right]\enspace.
    \end{align*}
    Solving the inner minimization problem for each fixed $\lambda \geq 0$ explicitly we derive $\Gamma_\lambda$ defined point-wise as
    \begin{align*}
        \Gamma_\lambda(\bsx) = \Top_{\bsp}(\bsx, k) \cap \Gamma^*_\lambda(\bsx)\enspace,
    \end{align*}
    where $\Top_{\bsp}(\bsx, k), \Gamma^*_\lambda(\bsx)$ are the Oracles defined in Lemma~\ref{lem:top_k_oracle} and~\ref{lem:penalized_oracle} respectively.
    Substituting into the objective function we get
    \begin{align*}
        (*) \geq 1 - \min_{\lambda \geq 0} \ens{\lambda \bar{k}  + \sum_{\ell = 1}^{k}\Exp_{\bsX}\left(p_{(\ell)}(\bsX) - \lambda\right)_+}\enspace.
    \end{align*}
    We remark that the minimization problem above is convex.
    Besides, we observe that under Assumption~\ref{ass:continuity_global} it holds that the sub-differential of $\lambda \to \Exp_{\bsX}\left(p_{(\ell)}(\bsX) - \lambda\right)_+$ is given by
    \begin{align*}
        \partial_{\lambda}\Exp_{\bsX}\left(p_{(\ell)}(\bsX) - \lambda\right)_+ = \{\Prob(p_{(\ell)}(\bsX) \geq \lambda)\}\enspace.
    \end{align*}
    Indeed, Assumption~\ref{ass:continuity_global} guarantees that $\Prob(p_{(\ell)}(\bsX) = \lambda) = 0$ for all $\lambda \geq 0$.
    Thus, under Assumption~\ref{ass:continuity_global} the KKT conditions of the minimization problem above are given by
    \begin{align*}
        \begin{cases}
            \sum_{\ell = 1}^{k} \Prob\left(p_{(\ell)}(\bsX) \geq \lambda\right) = \bar{k} - \tau\\
            \lambda, \tau \geq 0\\
            \lambda\tau = 0
        \end{cases}\enspace.
    \end{align*}
    Let $G_k(\cdot)$ be defined as
    \begin{align*}
        G_k(\cdot) = \sum_{\ell = 1}^{k} \Prob\left(p_{(\ell)}(\bsX) \geq \cdot\right)\enspace,
    \end{align*}
    and denote by $G_k^{-1}(\cdot)$ its generalized inverse.
    Continuity assumption~\ref{ass:continuity_global} guarantees that the function $G_k$ is continuous and thus we can solve the KKT conditions for $\lambda, \tau$ as
    \begin{align*}
    \begin{cases}
        \lambda = G_k^{-1}(\bar{k} - \tau)\enspace,\\
        \tau \geq 0\enspace,\\
        \lambda \tau = 0\enspace,
    \end{cases}
    \end{align*}
    We claim that if $\bar{k} < k$ then $G^{-1}_k(\bar k) \neq 0$.
    Assume the opposite. We clearly have $ G_k(0) = k$.
    Moreover, the fact that $\lambda \to G_k(\lambda)$
    is continuous non-increasing implies that for all $\bar k \in (0, k)$ we have
    \begin{align*}
       G_k(G^{-1}_k(\bar k)) = \bar k\enspace.
    \end{align*}
   By our assumption $G^{-1}_k(\bar k) = 0$ and thus
   \begin{align*}
       k = G_k(0) = G_k(G^{-1}_k(\bar k)) = \bar k\enspace.
   \end{align*}
    The above contradicts the fact that $\bar k < k$, hence $G^{-1}_k(\bar k) \neq 0$.

    Since $G^{-1}_k(\bar k) \neq 0$, then, complementary slackness condition implies that a solution is given by $\tau^* = 0$ and $\lambda^* = G^{-1}_k(\bar k)$.
    Also notice that the dual solution $\Gamma_{\lambda^*}(\bsx) = \Top_{\bsp}(\bsx, k) \cap \Gamma^*_{\lambda^*}(\bsx)$ is feasible for the primal problem.
    Indeed,
    \begin{align*}
        &|\Gamma_{\lambda^*}(\bsx)| \leq |\Top_{\bsp}(\bsx, k) | = k\\
        &\Exp|\Gamma_{\lambda^*}(\bsX)| = \sum_{\ell = 1}^k\Prob\left(p_{(\ell)}(\bsX) \geq G_k^{-1}(\bar k)\right) = G_k(G^{-1}_k(\bar k)) = \bar k\enspace.
    \end{align*}
    Thus, $\Gamma_{\lambda^*}(\bsx)$ is a solution of the primal problem.
\end{proof}

\begin{proof}[Proof of Lemma~\ref{lem:size_size_hybrid}]
    We start similarly to the proof of Lemma~\ref{lem:hybrid_size_control_oracle}.
    Fix $\bar{\epsilon} < \epsilon$, and set
    \begin{align*}
        (*) = \min_{\Gamma}\enscond{\Exp_{\bsX}|\Gamma(\bsX)|}{\Prob(Y \notin \Gamma(\bsX)) \leq \bar{\epsilon},\,\,\Prob(Y \notin \Gamma(\bsX) | \bsX) \leq \epsilon, \text{ a.s. } \Prob_{\bsX} }\enspace.
    \end{align*}
    Simple algebraic manipulations lead to
    \begin{align*}
        (*)
        &= \min_{\Gamma}\max_{\lambda \geq 0} \enscond{\lambda(1 - \bar{\epsilon}) + \Exp_{\bsX}\left[\sum_{\ell = 1}^L(1 - \lambda p_{\ell}(\bsX))\ind{\ell \in \Gamma(\bsX)}\right]}{\Prob(Y \notin \Gamma(\bsX) | \bsX) \leq \epsilon, \text{ a.s. } \Prob_{\bsX}}\\
        &\geq
        \max_{\lambda \geq 0}\min_{\Gamma} \enscond{\lambda(1 - \bar{\epsilon}) + \Exp_{\bsX}\left[\sum_{\ell = 1}^L(1 - \lambda p_{\ell}(\bsX))\ind{\ell \in \Gamma(\bsX)}\right]}{\Prob(Y \notin \Gamma(\bsX) \mid \bsX) \leq \epsilon, \text{ a.s. } \Prob_{\bsX}}\enspace.
    \end{align*}
    For each fixed $\lambda > 0$ a solution $\Gamma_\lambda$ of the inner minimization problem is given by
    \begin{align*}
        \Gamma_\lambda(\bsx) = \Gamma^*_\epsilon(\bsx) \cap \bar\Gamma(\bsx)\enspace,
    \end{align*}
    where $\Gamma^*_\epsilon$ is defined in Lemma~\ref{lem:hard_coverage_control_oracle} and $\bar\Gamma(\bsx)$ is defined as
    \begin{align*}
        \bar\Gamma(\bsx) = \enscond{\ell \in [L]}{p_{\ell}(\bsx) \geq 1 / \lambda}\enspace.
    \end{align*}
    Substituting this into the objective function we get
    \begin{align*}
        (*)
        &\geq
        \max_{\lambda \geq 0}\ens{\lambda(1 - \bar{\epsilon}) - \Exp_{\bsX}\left[\sum_{\ell = 1}^{k_\epsilon(\bsX)}\left(\lambda p_{(\ell)}(\bsX) - 1\right)_+\right]}\\
        &=
        -\min_{\lambda \geq 0}\ens{\Exp_{\bsX}\left[\sum_{\ell = 1}^{k_\epsilon(\bsX)}\left(\lambda p_{(\ell)}(\bsX) - 1\right)_+\right] - \lambda(1 - \bar{\epsilon})}\enspace.
    \end{align*}
    Under Assumption~\ref{ass:continuity_global} the KKT conditions are given by
    \begin{align*}
        &\Exp_{\bsX}\left[\sum_{\ell = 1}^{k_\epsilon(\bsX)}p_{(\ell)}(\bsX)\ind{ p_{(\ell)}(\bsX) \geq 1 / \lambda}\right] = 1 - \bar{\epsilon} - \tau\enspace,\\
        &\lambda, \tau \geq 0\enspace,\\
        &\lambda\tau = 0\enspace.
    \end{align*}
    Introduce
    \begin{align*}
       H_{\epsilon}(t) = \Exp_{\bsX}\left[\sum_{k = 1}^K\sum_{\ell = 1}^{k}p_{(\ell)}(\bsX)\ind{ p_{(\ell)}(\bsX) \geq t}\ind{k_{\epsilon}(\bsX) = k}\right]\enspace.
    \end{align*}
    Then, Assumption~\ref{ass:continuity_global} guarantees that
    \begin{align*}
         &\lambda = 1 / H_{\epsilon}^{-1}(1 - \bar{\epsilon} - \tau)\\
         &\lambda, \tau \geq 0\enspace,\\
         &\lambda\tau = 0\enspace.
    \end{align*}
    Notice that $H_{\epsilon}(0) \geq \Exp[p_{(1)}(\bsX)] \geq 1 / K$ and $H_{\epsilon}(1) = 0$.
    Hence, since $t \to H_{\epsilon}(t)$ is monotone continuous and $\Prob(p_{(1)}(\bsX) \geq 1 - \bar{\epsilon}) > 0$ it holds that $H_{\epsilon}^{-1}(1 - \bar{\epsilon}) > 0$.
\end{proof}

\bibliographystyle{imsart-nameyear}
\bibliography{bibliography.bib}

\begin{thebibliography}{61}

\bibitem[\protect\citeauthoryear{Barber et~al.}{2019}]{barber2019limits}
\begin{barticle}[author]
\bauthor{\bsnm{Barber},~\bfnm{Rina~Foygel}\binits{R.~F.}},
  \bauthor{\bsnm{Candes},~\bfnm{Emmanuel~J}\binits{E.~J.}},
  \bauthor{\bsnm{Ramdas},~\bfnm{Aaditya}\binits{A.}} \AND
  \bauthor{\bsnm{Tibshirani},~\bfnm{Ryan~J}\binits{R.~J.}}
(\byear{2019}).
\btitle{The limits of distribution-free conditional predictive inference}.
\bjournal{arXiv preprint arXiv:1903.04684}.
\end{barticle}
\endbibitem

\bibitem[\protect\citeauthoryear{Bartlett, Jordan and
  McAuliffe}{2006}]{Bartlett_Jordan_McAuliffe06}
\begin{barticle}[author]
\bauthor{\bsnm{Bartlett},~\bfnm{P.}\binits{P.}},
  \bauthor{\bsnm{Jordan},~\bfnm{M.}\binits{M.}} \AND
  \bauthor{\bsnm{McAuliffe},~\bfnm{J.}\binits{J.}}
(\byear{2006}).
\btitle{Convexity, classification, and risk bounds}.
\bjournal{Journal of the American Statistical Association}
\bvolume{101}
\bpages{138--156}.
\end{barticle}
\endbibitem

\bibitem[\protect\citeauthoryear{Berrada, Zisserman and
  Kumar}{2018}]{Berrada2018}
\begin{barticle}[author]
\bauthor{\bsnm{Berrada},~\bfnm{Leonard}\binits{L.}},
  \bauthor{\bsnm{Zisserman},~\bfnm{Andrew}\binits{A.}} \AND
  \bauthor{\bsnm{Kumar},~\bfnm{M~Pawan}\binits{M.~P.}}
(\byear{2018}).
\btitle{Smooth loss functions for deep top-k classification}.
\bjournal{ICLR}.
\end{barticle}
\endbibitem

\bibitem[\protect\citeauthoryear{Bhatia et~al.}{2016}]{Bhatia16}
\begin{bmisc}[author]
\bauthor{\bsnm{Bhatia},~\bfnm{K.}\binits{K.}},
  \bauthor{\bsnm{Dahiya},~\bfnm{K.}\binits{K.}},
  \bauthor{\bsnm{Jain},~\bfnm{H.}\binits{H.}},
  \bauthor{\bsnm{Mittal},~\bfnm{A.}\binits{A.}},
  \bauthor{\bsnm{Prabhu},~\bfnm{Y.}\binits{Y.}} \AND
  \bauthor{\bsnm{Varma},~\bfnm{M.}\binits{M.}}
(\byear{2016}).
\btitle{The extreme classification repository: Multi-label datasets and code}.
\end{bmisc}
\endbibitem

\bibitem[\protect\citeauthoryear{Cai, Low and Ma}{2014}]{Cai2014Confidence}
\begin{barticle}[author]
\bauthor{\bsnm{Cai},~\bfnm{T.~Tony}\binits{T.~T.}},
  \bauthor{\bsnm{Low},~\bfnm{Mark}\binits{M.}} \AND
  \bauthor{\bsnm{Ma},~\bfnm{Zongming}\binits{Z.}}
(\byear{2014}).
\btitle{{Adaptive Confidence Bands for Nonparametric Regression Functions}}.
\bjournal{Journal of the American Statistical Association}
\bvolume{109}
\bpages{1054-1070}.
\end{barticle}
\endbibitem

\bibitem[\protect\citeauthoryear{Champ et~al.}{2015}]{champ2015comparative}
\begin{binproceedings}[author]
\bauthor{\bsnm{Champ},~\bfnm{Julien}\binits{J.}},
  \bauthor{\bsnm{Lorieul},~\bfnm{Titouan}\binits{T.}},
  \bauthor{\bsnm{Servajean},~\bfnm{Maximilien}\binits{M.}} \AND
  \bauthor{\bsnm{Joly},~\bfnm{Alexis}\binits{A.}}
(\byear{2015}).
\btitle{A comparative study of fine-grained classification methods in the
  context of the LifeCLEF plant identification challenge 2015}.
In \bbooktitle{CLEF: Conference and Labs of the Evaluation Forum}
\bvolume{1391}.
\end{binproceedings}
\endbibitem

\bibitem[\protect\citeauthoryear{Champ et~al.}{2016}]{champ2016categorizing}
\begin{barticle}[author]
\bauthor{\bsnm{Champ},~\bfnm{Julien}\binits{J.}},
  \bauthor{\bsnm{Lorieul},~\bfnm{Titouan}\binits{T.}},
  \bauthor{\bsnm{Bonnet},~\bfnm{Pierre}\binits{P.}},
  \bauthor{\bsnm{Maghnaoui},~\bfnm{Najate}\binits{N.}},
  \bauthor{\bsnm{Sereno},~\bfnm{Christophe}\binits{C.}},
  \bauthor{\bsnm{Dessup},~\bfnm{Thierry}\binits{T.}},
  \bauthor{\bsnm{Boursiquot},~\bfnm{Jean-Michel}\binits{J.-M.}},
  \bauthor{\bsnm{Audeguin},~\bfnm{Laurent}\binits{L.}},
  \bauthor{\bsnm{Lacombe},~\bfnm{Thierry}\binits{T.}} \AND
  \bauthor{\bsnm{Joly},~\bfnm{Alexis}\binits{A.}}
(\byear{2016}).
\btitle{Categorizing plant images at the variety level: Did you say
  fine-grained?}
\bjournal{Pattern Recognition Letters}
\bvolume{81}
\bpages{71--79}.
\end{barticle}
\endbibitem

\bibitem[\protect\citeauthoryear{Chow}{1957}]{Chow57}
\begin{barticle}[author]
\bauthor{\bsnm{Chow},~\bfnm{{C. -K. }}\binits{C.}}
(\byear{1957}).
\btitle{An optimum character recognition system using decision functions}.
\bjournal{IRE Transactions on Electronic Computers}
\bvolume{4}
\bpages{247--254}.
\end{barticle}
\endbibitem

\bibitem[\protect\citeauthoryear{Chow}{1970}]{Chow70}
\begin{barticle}[author]
\bauthor{\bsnm{Chow},~\bfnm{C.}\binits{C.}}
(\byear{1970}).
\btitle{On optimum error and reject trade-off}.
\bjournal{IEEE Trans. Inform. Theory}
\bvolume{16}
\bpages{41--46}.
\end{barticle}
\endbibitem

\bibitem[\protect\citeauthoryear{Chzhen, Denis and Hebiri}{2019}]{Chzhen2019}
\begin{barticle}[author]
\bauthor{\bsnm{Chzhen},~\bfnm{Evgenii}\binits{E.}},
  \bauthor{\bsnm{Denis},~\bfnm{Christophe}\binits{C.}} \AND
  \bauthor{\bsnm{Hebiri},~\bfnm{Mohamed}\binits{M.}}
(\byear{2019}).
\btitle{Minimax semi-supervised confidence sets for multi-class
  classification}.
\bjournal{arXiv preprint arXiv:1904.12527}.
\end{barticle}
\endbibitem

\bibitem[\protect\citeauthoryear{Chzhen, Denis and
  Hebiri}{2021}]{chzhen2019minimax}
\begin{barticle}[author]
\bauthor{\bsnm{Chzhen},~\bfnm{Evgenii}\binits{E.}},
  \bauthor{\bsnm{Denis},~\bfnm{Christophe}\binits{C.}} \AND
  \bauthor{\bsnm{Hebiri},~\bfnm{Mohamed}\binits{M.}}
(\byear{2021}).
\btitle{Minimax semi-supervised set-valued approach to multi-class
  classification}.
\bjournal{Bernoulli}.
\end{barticle}
\endbibitem

\bibitem[\protect\citeauthoryear{Ciregan, Meier and
  Schmidhuber}{2012}]{ciregan2012multi}
\begin{binproceedings}[author]
\bauthor{\bsnm{Ciregan},~\bfnm{Dan}\binits{D.}},
  \bauthor{\bsnm{Meier},~\bfnm{Ueli}\binits{U.}} \AND
  \bauthor{\bsnm{Schmidhuber},~\bfnm{J{\"u}rgen}\binits{J.}}
(\byear{2012}).
\btitle{Multi-column deep neural networks for image classification}.
In \bbooktitle{2012 IEEE conference on computer vision and pattern recognition}
\bpages{3642--3649}.
\bpublisher{IEEE}.
\end{binproceedings}
\endbibitem

\bibitem[\protect\citeauthoryear{Dembczy{\'n}ski
  et~al.}{2012}]{dembczynski2012label}
\begin{barticle}[author]
\bauthor{\bsnm{Dembczy{\'n}ski},~\bfnm{Krzysztof}\binits{K.}},
  \bauthor{\bsnm{Waegeman},~\bfnm{Willem}\binits{W.}},
  \bauthor{\bsnm{Cheng},~\bfnm{Weiwei}\binits{W.}} \AND
  \bauthor{\bsnm{H{\"u}llermeier},~\bfnm{Eyke}\binits{E.}}
(\byear{2012}).
\btitle{On label dependence and loss minimization in multi-label
  classification}.
\bjournal{Machine Learning}
\bvolume{88}
\bpages{5--45}.
\end{barticle}
\endbibitem

\bibitem[\protect\citeauthoryear{Denis and Hebiri}{2017}]{Denis2017}
\begin{barticle}[author]
\bauthor{\bsnm{Denis},~\bfnm{Christophe}\binits{C.}} \AND
  \bauthor{\bsnm{Hebiri},~\bfnm{Mohamed}\binits{M.}}
(\byear{2017}).
\btitle{Confidence sets with expected sizes for multiclass classification}.
\bjournal{Journal of Machine Learning Research}
\bvolume{18}
\bpages{1-28}.
\end{barticle}
\endbibitem

\bibitem[\protect\citeauthoryear{Denis and Hebiri}{2020}]{denis2020consistency}
\begin{barticle}[author]
\bauthor{\bsnm{Denis},~\bfnm{Christophe}\binits{C.}} \AND
  \bauthor{\bsnm{Hebiri},~\bfnm{Mohamed}\binits{M.}}
(\byear{2020}).
\btitle{Consistency of plug-in confidence sets for classification in
  semi-supervised learning}.
\bjournal{Journal of Nonparametric Statistics}
\bvolume{32}
\bpages{42--72}.
\end{barticle}
\endbibitem

\bibitem[\protect\citeauthoryear{Elkan}{2001}]{elkan2001foundations}
\begin{binproceedings}[author]
\bauthor{\bsnm{Elkan},~\bfnm{Charles}\binits{C.}}
(\byear{2001}).
\btitle{The foundations of cost-sensitive learning}.
In \bbooktitle{Proceedings of the 17th international joint conference on
  Artificial intelligence-Volume 2}
\bpages{973--978}.
\end{binproceedings}
\endbibitem

\bibitem[\protect\citeauthoryear{Embrechts and
  Hofert}{2013}]{embrechts2013note}
\begin{barticle}[author]
\bauthor{\bsnm{Embrechts},~\bfnm{Paul}\binits{P.}} \AND
  \bauthor{\bsnm{Hofert},~\bfnm{Marius}\binits{M.}}
(\byear{2013}).
\btitle{A note on generalized inverses}.
\bjournal{Mathematical Methods of Operations Research}
\bvolume{77}
\bpages{423--432}.
\end{barticle}
\endbibitem

\bibitem[\protect\citeauthoryear{G{\"o}eau, Joly and
  Bonnet}{2015a}]{plantclef2015}
\begin{barticle}[author]
\bauthor{\bsnm{G{\"o}eau},~\bfnm{Herv{\'e}}\binits{H.}},
  \bauthor{\bsnm{Joly},~\bfnm{Alexis}\binits{A.}} \AND
  \bauthor{\bsnm{Bonnet},~\bfnm{Pierre}\binits{P.}}
(\byear{2015}a).
\btitle{{LifeCLEF} plant identification task 2015}.
\bjournal{CLEF working notes}
\bvolume{2015}.
\end{barticle}
\endbibitem

\bibitem[\protect\citeauthoryear{G{\"o}eau, Joly and Bonnet}{2015b}]{Goeau2015}
\begin{barticle}[author]
\bauthor{\bsnm{G{\"o}eau},~\bfnm{Herv{\'e}}\binits{H.}},
  \bauthor{\bsnm{Joly},~\bfnm{Alexis}\binits{A.}} \AND
  \bauthor{\bsnm{Bonnet},~\bfnm{Pierre}\binits{P.}}
(\byear{2015}b).
\btitle{{LifeCLEF} plant identification task 2015}.
\bjournal{CLEF working notes}
\bvolume{2015}.
\end{barticle}
\endbibitem

\bibitem[\protect\citeauthoryear{Grycko}{1993}]{grycko1993}
\begin{binproceedings}[author]
\bauthor{\bsnm{Grycko},~\bfnm{Eugen}\binits{E.}}
(\byear{1993}).
\btitle{Classification with Set-Valued Decision Functions}.
In \bbooktitle{Information and Classification}
(\beditor{\bfnm{Otto}\binits{O.}~\bsnm{Opitz}},
  \beditor{\bfnm{Berthold}\binits{B.}~\bsnm{Lausen}} \AND
  \beditor{\bfnm{R{\"u}diger}\binits{R.}~\bsnm{Klar}}, eds.)
\bpages{218--224}.
\bpublisher{Springer Berlin Heidelberg}, \baddress{Berlin, Heidelberg}.
\end{binproceedings}
\endbibitem

\bibitem[\protect\citeauthoryear{Guo et~al.}{2017}]{Guo2017}
\begin{binproceedings}[author]
\bauthor{\bsnm{Guo},~\bfnm{Chuan}\binits{C.}},
  \bauthor{\bsnm{Pleiss},~\bfnm{Geoff}\binits{G.}},
  \bauthor{\bsnm{Sun},~\bfnm{Yu}\binits{Y.}} \AND
  \bauthor{\bsnm{Weinberger},~\bfnm{Kilian~Q}\binits{K.~Q.}}
(\byear{2017}).
\btitle{On Calibration of Modern Neural Networks}.
In \bbooktitle{International Conference on Machine Learning}
\bpages{1321--1330}.
\end{binproceedings}
\endbibitem

\bibitem[\protect\citeauthoryear{Gy\"ofi and Walk}{2020}]{Gyorfi_Walk_20}
\begin{barticle}[author]
\bauthor{\bsnm{Gy\"ofi},~\bfnm{L.}\binits{L.}} \AND
  \bauthor{\bsnm{Walk},~\bfnm{H.}\binits{H.}}
(\byear{2020}).
\btitle{Nearest neighbor based conformal prediction}.
\bjournal{Submitted}.
\end{barticle}
\endbibitem

\bibitem[\protect\citeauthoryear{Ha}{1996}]{Ha1996}
\begin{binproceedings}[author]
\bauthor{\bsnm{Ha},~\bfnm{Thien~M}\binits{T.~M.}}
(\byear{1996}).
\btitle{An optimum class-selective rejection rule for pattern recognition}.
In \bbooktitle{Proceedings of 13th International Conference on Pattern
  Recognition}
\bvolume{2}
\bpages{75--80}.
\bpublisher{IEEE}.
\end{binproceedings}
\endbibitem

\bibitem[\protect\citeauthoryear{Ha}{1997a}]{Ha1997}
\begin{barticle}[author]
\bauthor{\bsnm{Ha},~\bfnm{Thien~M}\binits{T.~M.}}
(\byear{1997}a).
\btitle{The optimum class-selective rejection rule}.
\bjournal{PAMI}.
\end{barticle}
\endbibitem

\bibitem[\protect\citeauthoryear{Ha}{1997b}]{Ha1997b}
\begin{barticle}[author]
\bauthor{\bsnm{Ha},~\bfnm{Thien~M}\binits{T.~M.}}
(\byear{1997}b).
\btitle{Optimum tradeoff between class-selective rejection error and average
  number of classes}.
\bjournal{Engineering Applications of Artificial Intelligence}
\bvolume{10}
\bpages{525--529}.
\end{barticle}
\endbibitem

\bibitem[\protect\citeauthoryear{Hastie and
  Fithian}{2013}]{hastie2013inference}
\begin{barticle}[author]
\bauthor{\bsnm{Hastie},~\bfnm{Trevor}\binits{T.}} \AND
  \bauthor{\bsnm{Fithian},~\bfnm{Will}\binits{W.}}
(\byear{2013}).
\btitle{Inference from presence-only data; the ongoing controversy}.
\bjournal{Ecography}
\bvolume{36}
\bpages{864--867}.
\end{barticle}
\endbibitem

\bibitem[\protect\citeauthoryear{He et~al.}{2015}]{He2016}
\begin{barticle}[author]
\bauthor{\bsnm{He},~\bfnm{Kaiming}\binits{K.}},
  \bauthor{\bsnm{Zhang},~\bfnm{Xiangyu}\binits{X.}},
  \bauthor{\bsnm{Ren},~\bfnm{Shaoqing}\binits{S.}} \AND
  \bauthor{\bsnm{Sun},~\bfnm{Jian}\binits{J.}}
(\byear{2015}).
\btitle{Deep Residual Learning for Image Recognition}.
\bjournal{CoRR}
\bvolume{abs/1512.03385}.
\end{barticle}
\endbibitem

\bibitem[\protect\citeauthoryear{Herbei and Wegkamp}{2006}]{Herbei_Wegkamp06}
\begin{barticle}[author]
\bauthor{\bsnm{Herbei},~\bfnm{R.}\binits{R.}} \AND
  \bauthor{\bsnm{Wegkamp},~\bfnm{M.}\binits{M.}}
(\byear{2006}).
\btitle{Classification with reject option}.
\bjournal{Canad. J. Statist.}
\bvolume{34}
\bpages{709--721}.
\end{barticle}
\endbibitem

\bibitem[\protect\citeauthoryear{Kourou et~al.}{2015}]{kourou2015machine}
\begin{barticle}[author]
\bauthor{\bsnm{Kourou},~\bfnm{Konstantina}\binits{K.}},
  \bauthor{\bsnm{Exarchos},~\bfnm{Themis~P}\binits{T.~P.}},
  \bauthor{\bsnm{Exarchos},~\bfnm{Konstantinos~P}\binits{K.~P.}},
  \bauthor{\bsnm{Karamouzis},~\bfnm{Michalis~V}\binits{M.~V.}} \AND
  \bauthor{\bsnm{Fotiadis},~\bfnm{Dimitrios~I}\binits{D.~I.}}
(\byear{2015}).
\btitle{Machine learning applications in cancer prognosis and prediction}.
\bjournal{Computational and structural biotechnology journal}
\bvolume{13}
\bpages{8--17}.
\end{barticle}
\endbibitem

\bibitem[\protect\citeauthoryear{Lambin et~al.}{2017}]{lambin2017decision}
\begin{barticle}[author]
\bauthor{\bsnm{Lambin},~\bfnm{Philippe}\binits{P.}},
  \bauthor{\bsnm{Zindler},~\bfnm{Jaap}\binits{J.}},
  \bauthor{\bsnm{Vanneste},~\bfnm{Ben~GL}\binits{B.~G.}}, \bauthor{\bsnm{Van
  De~Voorde},~\bfnm{Lien}\binits{L.}},
  \bauthor{\bsnm{Eekers},~\bfnm{Danielle}\binits{D.}},
  \bauthor{\bsnm{Compter},~\bfnm{Inge}\binits{I.}},
  \bauthor{\bsnm{Panth},~\bfnm{Kranthi~Marella}\binits{K.~M.}},
  \bauthor{\bsnm{Peerlings},~\bfnm{Jurgen}\binits{J.}},
  \bauthor{\bsnm{Larue},~\bfnm{Ruben~THM}\binits{R.~T.}},
  \bauthor{\bsnm{Deist},~\bfnm{Timo~M}\binits{T.~M.}} \betal{et~al.}
(\byear{2017}).
\btitle{Decision support systems for personalized and participative radiation
  oncology}.
\bjournal{Advanced drug delivery reviews}
\bvolume{109}
\bpages{131--153}.
\end{barticle}
\endbibitem

\bibitem[\protect\citeauthoryear{Lapin, Hein and Schiele}{2015}]{lapin2015top}
\begin{binproceedings}[author]
\bauthor{\bsnm{Lapin},~\bfnm{Maksim}\binits{M.}},
  \bauthor{\bsnm{Hein},~\bfnm{Matthias}\binits{M.}} \AND
  \bauthor{\bsnm{Schiele},~\bfnm{Bernt}\binits{B.}}
(\byear{2015}).
\btitle{Top-k multiclass SVM}.
In \bbooktitle{Advances in Neural Information Processing Systems}
\bpages{325--333}.
\end{binproceedings}
\endbibitem

\bibitem[\protect\citeauthoryear{Lapin, Hein and Schiele}{2016}]{Lapin2016}
\begin{binproceedings}[author]
\bauthor{\bsnm{Lapin},~\bfnm{Maksim}\binits{M.}},
  \bauthor{\bsnm{Hein},~\bfnm{Matthias}\binits{M.}} \AND
  \bauthor{\bsnm{Schiele},~\bfnm{Bernt}\binits{B.}}
(\byear{2016}).
\btitle{Loss functions for top-k error: Analysis and insights}.
In \bbooktitle{CVPR}.
\end{binproceedings}
\endbibitem

\bibitem[\protect\citeauthoryear{Le~Capitaine}{2014}]{LeCapitaine2014}
\begin{barticle}[author]
\bauthor{\bsnm{Le~Capitaine},~\bfnm{Hoel}\binits{H.}}
(\byear{2014}).
\btitle{A unified view of class-selection with probabilistic classifiers}.
\bjournal{Pattern recognition}
\bvolume{47}
\bpages{843--853}.
\end{barticle}
\endbibitem

\bibitem[\protect\citeauthoryear{LeCun et~al.}{1998}]{LeCun1998}
\begin{barticle}[author]
\bauthor{\bsnm{LeCun},~\bfnm{Yann}\binits{Y.}},
  \bauthor{\bsnm{Bottou},~\bfnm{L{\'e}on}\binits{L.}},
  \bauthor{\bsnm{Bengio},~\bfnm{Yoshua}\binits{Y.}} \AND
  \bauthor{\bsnm{Haffner},~\bfnm{Patrick}\binits{P.}}
(\byear{1998}).
\btitle{Gradient-based learning applied to document recognition}.
\bjournal{Proceedings of the IEEE}
\bvolume{86}
\bpages{2278--2324}.
\end{barticle}
\endbibitem

\bibitem[\protect\citeauthoryear{Lei}{2014}]{Lei2014}
\begin{barticle}[author]
\bauthor{\bsnm{Lei},~\bfnm{Jing}\binits{J.}}
(\byear{2014}).
\btitle{Classification with Confidence}.
\bjournal{Biometrika}.
\end{barticle}
\endbibitem

\bibitem[\protect\citeauthoryear{Lei and Wasserman}{2014}]{Lei_Wasserman2014}
\begin{barticle}[author]
\bauthor{\bsnm{Lei},~\bfnm{Jing}\binits{J.}} \AND
  \bauthor{\bsnm{Wasserman},~\bfnm{Larry}\binits{L.}}
(\byear{2014}).
\btitle{{Distribution-free prediction bands for non-parametric regression}}.
\bjournal{Journal of the Royal Statistical Society Series B}
\bvolume{76}
\bpages{71-96}.
\end{barticle}
\endbibitem

\bibitem[\protect\citeauthoryear{Lei
  et~al.}{2018}]{LeiGsellRinaldoTibshiraniWasserman_16}
\begin{barticle}[author]
\bauthor{\bsnm{Lei},~\bfnm{J.}\binits{J.}},
  \bauthor{\bsnm{G'Sell},~\bfnm{M.}\binits{M.}},
  \bauthor{\bsnm{Rinaldo},~\bfnm{A.}\binits{A.}},
  \bauthor{\bsnm{Tibshirani},~\bfnm{R.}\binits{R.}} \AND
  \bauthor{\bsnm{Wasserman},~\bfnm{L.}\binits{L.}}
(\byear{2018}).
\btitle{Distribution-free predictive inference for regression}.
\bjournal{J. Amer. Statist. Assoc.}
\bvolume{113}
\bpages{1094--1111}.
\end{barticle}
\endbibitem

\bibitem[\protect\citeauthoryear{Ling, Huang and Zhang}{2003}]{ling2003auc}
\begin{binproceedings}[author]
\bauthor{\bsnm{Ling},~\bfnm{Charles~X}\binits{C.~X.}},
  \bauthor{\bsnm{Huang},~\bfnm{Jin}\binits{J.}} \AND
  \bauthor{\bsnm{Zhang},~\bfnm{Harry}\binits{H.}}
(\byear{2003}).
\btitle{AUC: a better measure than accuracy in comparing learning algorithms}.
In \bbooktitle{Conference of the canadian society for computational studies of
  intelligence}
\bpages{329--341}.
\bpublisher{Springer}.
\end{binproceedings}
\endbibitem

\bibitem[\protect\citeauthoryear{Lorieul, Joly and Shasha}{2020}]{Titouan20}
\begin{barticle}[author]
\bauthor{\bsnm{Lorieul},~\bfnm{Titouan}\binits{T.}},
  \bauthor{\bsnm{Joly},~\bfnm{Alexis}\binits{A.}} \AND
  \bauthor{\bsnm{Shasha},~\bfnm{Dennis}\binits{D.}}
(\byear{2020}).
\btitle{Average-K classification: when and how to predict adaptive confidence
  sets rather than top-K}.
\end{barticle}
\endbibitem

\bibitem[\protect\citeauthoryear{Mac~Aodha, Cole and
  Perona}{2019}]{mac2019presence}
\begin{binproceedings}[author]
\bauthor{\bsnm{Mac~Aodha},~\bfnm{Oisin}\binits{O.}},
  \bauthor{\bsnm{Cole},~\bfnm{Elijah}\binits{E.}} \AND
  \bauthor{\bsnm{Perona},~\bfnm{Pietro}\binits{P.}}
(\byear{2019}).
\btitle{Presence-only geographical priors for fine-grained image
  classification}.
In \bbooktitle{Proceedings of the IEEE International Conference on Computer
  Vision}
\bpages{9596--9606}.
\end{binproceedings}
\endbibitem

\bibitem[\protect\citeauthoryear{Ni et~al.}{2019}]{Ni19}
\begin{bincollection}[author]
\bauthor{\bsnm{Ni},~\bfnm{C.}\binits{C.}},
  \bauthor{\bsnm{Charoenphakdee},~\bfnm{N.}\binits{N.}},
  \bauthor{\bsnm{Honda},~\bfnm{J.}\binits{J.}} \AND
  \bauthor{\bsnm{Sugiyama},~\bfnm{M.}\binits{M.}}
(\byear{2019}).
\btitle{On the Calibration of Multiclass Classification with Rejection}.
In \bbooktitle{Advances in Neural Information Processing Systems 32}
\bpages{2586--2596}.
\end{bincollection}
\endbibitem

\bibitem[\protect\citeauthoryear{Ramaswamy
  et~al.}{2018}]{ramaswamy2018consistent}
\begin{barticle}[author]
\bauthor{\bsnm{Ramaswamy},~\bfnm{Harish~G}\binits{H.~G.}},
  \bauthor{\bsnm{Tewari},~\bfnm{Ambuj}\binits{A.}},
  \bauthor{\bsnm{Agarwal},~\bfnm{Shivani}\binits{S.}} \betal{et~al.}
(\byear{2018}).
\btitle{Consistent algorithms for multiclass classification with an abstain
  option}.
\bjournal{Electronic Journal of Statistics}
\bvolume{12}
\bpages{530--554}.
\end{barticle}
\endbibitem

\bibitem[\protect\citeauthoryear{Romano, Sesia and
  Cand\`es}{2020}]{Romano_Sesia_Candes20}
\begin{barticle}[author]
\bauthor{\bsnm{Romano},~\bfnm{Y.}\binits{Y.}},
  \bauthor{\bsnm{Sesia},~\bfnm{M}\binits{M.}} \AND
  \bauthor{\bsnm{Cand\`es},~\bfnm{E.~J}\binits{E.~J.}}
(\byear{2020}).
\btitle{Classification with Valid and Adaptive Coverage}.
\bjournal{arXiv preprint arXiv:2006.02544}.
\end{barticle}
\endbibitem

\bibitem[\protect\citeauthoryear{Royle et~al.}{2012}]{royle2012likelihood}
\begin{barticle}[author]
\bauthor{\bsnm{Royle},~\bfnm{J~Andrew}\binits{J.~A.}},
  \bauthor{\bsnm{Chandler},~\bfnm{Richard~B}\binits{R.~B.}},
  \bauthor{\bsnm{Yackulic},~\bfnm{Charles}\binits{C.}} \AND
  \bauthor{\bsnm{Nichols},~\bfnm{James~D}\binits{J.~D.}}
(\byear{2012}).
\btitle{Likelihood analysis of species occurrence probability from
  presence-only data for modelling species distributions}.
\bjournal{Methods in Ecology and Evolution}
\bvolume{3}
\bpages{545--554}.
\end{barticle}
\endbibitem

\bibitem[\protect\citeauthoryear{Russakovsky et~al.}{2015a}]{imagenet}
\begin{barticle}[author]
\bauthor{\bsnm{Russakovsky},~\bfnm{Olga}\binits{O.}},
  \bauthor{\bsnm{Deng},~\bfnm{Jia}\binits{J.}},
  \bauthor{\bsnm{Su},~\bfnm{Hao}\binits{H.}},
  \bauthor{\bsnm{Krause},~\bfnm{Jonathan}\binits{J.}},
  \bauthor{\bsnm{Satheesh},~\bfnm{Sanjeev}\binits{S.}},
  \bauthor{\bsnm{Ma},~\bfnm{Sean}\binits{S.}},
  \bauthor{\bsnm{Huang},~\bfnm{Zhiheng}\binits{Z.}},
  \bauthor{\bsnm{Karpathy},~\bfnm{Andrej}\binits{A.}},
  \bauthor{\bsnm{Khosla},~\bfnm{Aditya}\binits{A.}},
  \bauthor{\bsnm{Bernstein},~\bfnm{Michael}\binits{M.}} \betal{et~al.}
(\byear{2015}a).
\btitle{Imagenet large scale visual recognition challenge}.
\bjournal{International journal of computer vision}
\bvolume{115}
\bpages{211--252}.
\end{barticle}
\endbibitem

\bibitem[\protect\citeauthoryear{Russakovsky et~al.}{2015b}]{Russakovsky2015}
\begin{barticle}[author]
\bauthor{\bsnm{Russakovsky},~\bfnm{Olga}\binits{O.}},
  \bauthor{\bsnm{Deng},~\bfnm{Jia}\binits{J.}},
  \bauthor{\bsnm{Su},~\bfnm{Hao}\binits{H.}},
  \bauthor{\bsnm{Krause},~\bfnm{Jonathan}\binits{J.}},
  \bauthor{\bsnm{Satheesh},~\bfnm{Sanjeev}\binits{S.}},
  \bauthor{\bsnm{Ma},~\bfnm{Sean}\binits{S.}},
  \bauthor{\bsnm{Huang},~\bfnm{Zhiheng}\binits{Z.}},
  \bauthor{\bsnm{Karpathy},~\bfnm{Andrej}\binits{A.}},
  \bauthor{\bsnm{Khosla},~\bfnm{Aditya}\binits{A.}},
  \bauthor{\bsnm{Bernstein},~\bfnm{Michael}\binits{M.}} \betal{et~al.}
(\byear{2015}b).
\btitle{{ImageNet} large scale visual recognition challenge}.
\bjournal{International journal of computer vision}
\bvolume{115}
\bpages{211--252}.
\end{barticle}
\endbibitem

\bibitem[\protect\citeauthoryear{Sadinle, Lei and
  Wasserman}{2019a}]{Sadinle2019}
\begin{barticle}[author]
\bauthor{\bsnm{Sadinle},~\bfnm{Mauricio}\binits{M.}},
  \bauthor{\bsnm{Lei},~\bfnm{Jing}\binits{J.}} \AND
  \bauthor{\bsnm{Wasserman},~\bfnm{Larry}\binits{L.}}
(\byear{2019}a).
\btitle{Least ambiguous set-valued classifiers with bounded error levels}.
\bjournal{Journal of the American Statistical Association}
\bvolume{114}
\bpages{223--234}.
\end{barticle}
\endbibitem

\bibitem[\protect\citeauthoryear{Sadinle, Lei and
  Wasserman}{2019b}]{sadinle2019least}
\begin{barticle}[author]
\bauthor{\bsnm{Sadinle},~\bfnm{Mauricio}\binits{M.}},
  \bauthor{\bsnm{Lei},~\bfnm{Jing}\binits{J.}} \AND
  \bauthor{\bsnm{Wasserman},~\bfnm{Larry}\binits{L.}}
(\byear{2019}b).
\btitle{Least ambiguous set-valued classifiers with bounded error levels}.
\bjournal{Journal of the American Statistical Association}
\bvolume{114}
\bpages{223--234}.
\end{barticle}
\endbibitem

\bibitem[\protect\citeauthoryear{Turney}{1994}]{turney1994cost}
\begin{barticle}[author]
\bauthor{\bsnm{Turney},~\bfnm{Peter~D}\binits{P.~D.}}
(\byear{1994}).
\btitle{Cost-sensitive classification: Empirical evaluation of a hybrid genetic
  decision tree induction algorithm}.
\bjournal{Journal of artificial intelligence research}
\bvolume{2}
\bpages{369--409}.
\end{barticle}
\endbibitem

\bibitem[\protect\citeauthoryear{Vapnik}{1998}]{Vapnik98}
\begin{bbook}[author]
\bauthor{\bsnm{Vapnik},~\bfnm{V.}\binits{V.}}
(\byear{1998}).
\btitle{Statistical learning theory}.
\bpublisher{Wiley}.
\end{bbook}
\endbibitem

\bibitem[\protect\citeauthoryear{Vovk}{2013}]{vovk_13}
\begin{barticle}[author]
\bauthor{\bsnm{Vovk},~\bfnm{Vladimir}\binits{V.}}
(\byear{2013}).
\btitle{Conditional validity of inductive conformal predictors}.
\bjournal{Mach. Learn.}
\bvolume{92}
\bpages{349--376}.
\end{barticle}
\endbibitem

\bibitem[\protect\citeauthoryear{Vovk, Gammerman and
  Shafer}{2005a}]{vovk2005algorithmic}
\begin{bbook}[author]
\bauthor{\bsnm{Vovk},~\bfnm{Vladimir}\binits{V.}},
  \bauthor{\bsnm{Gammerman},~\bfnm{Alex}\binits{A.}} \AND
  \bauthor{\bsnm{Shafer},~\bfnm{Glenn}\binits{G.}}
(\byear{2005}a).
\btitle{Algorithmic learning in a random world}.
\bpublisher{Springer Science \& Business Media}.
\end{bbook}
\endbibitem

\bibitem[\protect\citeauthoryear{Vovk, Gammerman and Shafer}{2005b}]{Vovk2005}
\begin{bbook}[author]
\bauthor{\bsnm{Vovk},~\bfnm{Vladimir}\binits{V.}},
  \bauthor{\bsnm{Gammerman},~\bfnm{Alex}\binits{A.}} \AND
  \bauthor{\bsnm{Shafer},~\bfnm{Glenn}\binits{G.}}
(\byear{2005}b).
\btitle{Algorithmic learning in a random world}.
\bpublisher{Springer Science \& Business Media}.
\end{bbook}
\endbibitem

\bibitem[\protect\citeauthoryear{Vovk et~al.}{2017}]{vovk2017}
\begin{barticle}[author]
\bauthor{\bsnm{Vovk},~\bfnm{Vladimir}\binits{V.}},
  \bauthor{\bsnm{Nouretdinov},~\bfnm{Ilia}\binits{I.}},
  \bauthor{\bsnm{Fedorova},~\bfnm{Valentina}\binits{V.}},
  \bauthor{\bsnm{Petej},~\bfnm{Ivan}\binits{I.}} \AND
  \bauthor{\bsnm{Gammerman},~\bfnm{Alex}\binits{A.}}
(\byear{2017}).
\btitle{Criteria of efficiency for set-valued classification}.
\bjournal{Annals of Mathematics and Artificial Intelligence}
\bvolume{81}
\bpages{21-47}.
\end{barticle}
\endbibitem

\bibitem[\protect\citeauthoryear{{Wu}, {Jia} and {Chen}}{2007}]{Wu07}
\begin{binproceedings}[author]
\bauthor{\bsnm{{Wu}},~\bfnm{Q.}\binits{Q.}},
  \bauthor{\bsnm{{Jia}},~\bfnm{C.}\binits{C.}} \AND
  \bauthor{\bsnm{{Chen}},~\bfnm{W.}\binits{W.}}
(\byear{2007}).
\btitle{A Novel Classification-Rejection Sphere SVMs for Multi-class
  Classification Problems}.
In \bbooktitle{Third International Conference on Natural Computation (ICNC
  2007)}
\bvolume{1}
\bpages{34-38}.
\end{binproceedings}
\endbibitem

\bibitem[\protect\citeauthoryear{Xie et~al.}{2017}]{xie2017aggregated}
\begin{binproceedings}[author]
\bauthor{\bsnm{Xie},~\bfnm{Saining}\binits{S.}},
  \bauthor{\bsnm{Girshick},~\bfnm{Ross}\binits{R.}},
  \bauthor{\bsnm{Doll{\'a}r},~\bfnm{Piotr}\binits{P.}},
  \bauthor{\bsnm{Tu},~\bfnm{Zhuowen}\binits{Z.}} \AND
  \bauthor{\bsnm{He},~\bfnm{Kaiming}\binits{K.}}
(\byear{2017}).
\btitle{Aggregated residual transformations for deep neural networks}.
In \bbooktitle{Proceedings of the IEEE conference on computer vision and
  pattern recognition}
\bpages{1492--1500}.
\end{binproceedings}
\endbibitem

\bibitem[\protect\citeauthoryear{Yang and Koyejo}{2020}]{Yang2020}
\begin{binproceedings}[author]
\bauthor{\bsnm{Yang},~\bfnm{Forest}\binits{F.}} \AND
  \bauthor{\bsnm{Koyejo},~\bfnm{Sanmi}\binits{S.}}
(\byear{2020}).
\btitle{On the consistency of top-k surrogate losses}.
In \bbooktitle{International Conference on Machine Learning}.
\end{binproceedings}
\endbibitem

\bibitem[\protect\citeauthoryear{Zhang}{2004}]{zhang2004statistical}
\begin{barticle}[author]
\bauthor{\bsnm{Zhang},~\bfnm{Tong}\binits{T.}}
(\byear{2004}).
\btitle{Statistical behavior and consistency of classification methods based on
  convex risk minimization}.
\bjournal{Annals of Statistics}
\bpages{56--85}.
\end{barticle}
\endbibitem

\bibitem[\protect\citeauthoryear{Zhang, Wang and Qiao}{2018}]{Zhang18}
\begin{barticle}[author]
\bauthor{\bsnm{Zhang},~\bfnm{Chong}\binits{C.}},
  \bauthor{\bsnm{Wang},~\bfnm{Wenbo}\binits{W.}} \AND
  \bauthor{\bsnm{Qiao},~\bfnm{Xingye}\binits{X.}}
(\byear{2018}).
\btitle{On reject and refine options in multicategory classification}.
\bjournal{J. Amer. Statist. Assoc.}
\bvolume{113}
\bpages{730--745}.
\end{barticle}
\endbibitem

\bibitem[\protect\citeauthoryear{Zhang and Zhou}{2013}]{zhang2013review}
\begin{barticle}[author]
\bauthor{\bsnm{Zhang},~\bfnm{Min-Ling}\binits{M.-L.}} \AND
  \bauthor{\bsnm{Zhou},~\bfnm{Zhi-Hua}\binits{Z.-H.}}
(\byear{2013}).
\btitle{A review on multi-label learning algorithms}.
\bjournal{IEEE transactions on knowledge and data engineering}
\bvolume{26}
\bpages{1819--1837}.
\end{barticle}
\endbibitem

\bibitem[\protect\citeauthoryear{Zubiaga}{2012}]{zubiaga2012enhancing}
\begin{barticle}[author]
\bauthor{\bsnm{Zubiaga},~\bfnm{Arkaitz}\binits{A.}}
(\byear{2012}).
\btitle{Enhancing navigation on wikipedia with social tags}.
\bjournal{arXiv preprint arXiv:1202.5469}.
\end{barticle}
\endbibitem

\end{thebibliography}
\newpage
\appendix


\end{document}